\documentclass{article}

\PassOptionsToPackage{numbers, compress}{natbib}

\usepackage[preprint]{neurips_2022}

\usepackage[utf8]{inputenc} 
\usepackage[T1]{fontenc}    
\usepackage{hyperref}       
\usepackage{url}            
\usepackage{booktabs}       
\usepackage{amsfonts}       
\usepackage{nicefrac}       
\usepackage{microtype}      
\usepackage{xcolor}         
\usepackage{amsmath}
\usepackage{ltl}

\usepackage{wrapfig}

\usepackage{caption}
\usepackage{subcaption}

\usepackage{amsmath}
\usepackage{amssymb}
\usepackage{mathtools}
\usepackage{amsthm}

\newtheorem{definition}{Definition}
\newtheorem{proposition}{Proposition}
\newtheorem{corollary}{Corollary}
\newtheorem{property}{Property}
\usepackage{amsmath}

\usepackage{amsmath,amsfonts,bm}

\def\eqref#1{equation~\ref{#1}}

\def\1{\bm{1}}

\DeclareMathAlphabet{\mathsfit}{\encodingdefault}{\sfdefault}{m}{sl}
\SetMathAlphabet{\mathsfit}{bold}{\encodingdefault}{\sfdefault}{bx}{n}
\newcommand{\tens}[1]{\bm{\mathsfit{#1}}}
\def\tA{{\tens{A}}}

\usepackage[export]{adjustbox}

\usepackage{float}
\usepackage{soul}

\title{Attention Flows for General Transformers}

\author{%
  Niklas Metzger\thanks{Equal contribution. Random order.}\\
  \small{CISPA Helmholtz Center for Information Security}\\
  \small{Saarbr\"ucken, Germany}\\
  \small{\texttt{niklas.metzger@cispa.de}}
  \And
  Christopher Hahn\footnotemark[1]\\
  \small{CISPA Helmholtz Center for Information Security}\\
  \small{Saarbr\"ucken, Germany}\\
  \small{\texttt{christopher.hahn@cispa.de}}
  \And
  Julian Siber\footnotemark[1]\\
  \small{CISPA Helmholtz Center for Information Security}\\
  \small{Saarbr\"ucken, Germany}\\
  \small{\texttt{julian.siber@cispa.de}}
  \And
  Frederik Schmitt\footnotemark[1]\\
  \small{CISPA Helmholtz Center for Information Security}\\
  \small{Saarbr\"ucken, Germany}\\
  \small{\texttt{frederik.schmitt@cispa.de}}
  \And
  Bernd Finkbeiner\\
  \small{CISPA Helmholtz Center for Information Security}\\
  \small{Saarbr\"ucken, Germany}\\
  \small{\texttt{finkbeiner@cispa.de}}
}

\begin{document}

\maketitle

\begin{abstract}
In this paper, we study the computation of how much an input token in a Transformer model influences its prediction.
We formalize a method to construct a flow network out of the attention values of encoder-only Transformer models and extend it to general Transformer architectures including an auto-regressive decoder.
We show that running a maxflow algorithm on the flow network construction yields Shapley values, which determine the impact of a player in cooperative game theory.
By interpreting the input tokens in the flow network as players, we can compute their influence on the total attention flow leading to the decoder's decision.
Additionally, we provide a library that computes and visualizes the attention flow of arbitrary Transformer models.
We show the usefulness of our implementation on various models trained on natural language processing and reasoning tasks.
\end{abstract}

\section{Introduction}
\label{sec:intro}
\emph{Transformers}~\citep{DBLP:conf/nips/VaswaniSPUJGKP17} are the dominant machine learning architecture of the recent years finding application in NLP (e.g., BERT~\citep{DBLP:conf/naacl/DevlinCLT19}, GPT-3~\citep{DBLP:conf/nips/BrownMRSKDNSSAA20}, or LaMDA~\citep{LaMBDA}), computer vision (see~\citet{khan2021transformers} for a survey), mathematical reasoning~\citep{lample2019deep,han2021proof}, or even code and hardware synthesis~\citep{chen2021evaluating,DBLP:journals/corr/abs-2107-11864}.
The Transformer architecture relies on an \emph{attention mechanism}~\citep{DBLP:journals/corr/BahdanauCB14} that mimics cognitive attention, which sets the focus of computation on a few concepts at a time.
In this paper, we rigorously formalize the method of constructing a flow network out of Transformer attention values~\citep{DBLP:conf/acl/AbnarZ20} and generalize it to models including a decoder.
While theoretically yielding a Shapley value~\citep{shapley201617} quite trivially, we show that this indeed results in meaningingful explanations for the influence of input tokens to the total flow effecting a decoder's prediction.

Their applicability in a wide range of domains has made the Transformer architecture incredibly popular.
Models are easily accessible for developers around the world. For example at \url{huggingface.co}~\citep{wolf2019huggingface}.
Blindly using or fine tuning these models, however, might lead to mispredictions and unwanted biases, which will have a considerable negative effect in their application domains.
The sheer size of the Transformer models makes it impossible to analyze
the networks by hand.
Explainability and visualization methods, e.g.,~\cite{vig2019multiscale}, aid the machine learning practitioner and researcher in finding the cause for a misprediction or reveal unwanted biases. The training method or the dataset can then be adjusted accordingly.

\citet{DBLP:conf/acl/AbnarZ20} introduced \emph{Attention Flow} as a post-processing interpretability technique that treats the self-attention weight matrices of a Transformer encoder as a flow network.
This technique allows for the analysis of the flow of attention through the Transformer encoder:
Computing the maxflow for an input token determines the impact of this token to the total attention flow.
\citet{DBLP:conf/acl/EthayarajhJ20} discussed a possible relation of the maxflow computation through the encoder flow network to Shapley values, which is a concept determining the impact of a player in cooperative game theory, and can be applied to measure the importance of a model's input features. However, the lack of a clear formalization of the underlying flow network has made it difficult to assess the validity of their claims, which we aim to adress in this work.

We extend our formalization of the approach to a Transformer-model-agnostic technique including general encoder-decoder Transformers and decoder-only Transformers such as GPT models~\cite{radford2018improving}.
While, after applying a positional encoding, the encoder processes the input tokens as a whole, the decoder layers operate auto-regressively, i.e., a sequence of tokens will be predicted step-by-step, and already predicted input tokens will be given as an input to the decoder.
This results in a significantly different shape of the flow network and, in particular, requires normalization to account for the bias towards tokens that were predicted later than others.
We account for the auto-regressive nature of the Transformer decoder by ensuring \emph{positional independence} of the computed maxflow values.
We implemented our constructions as a Python library, which we will publish under the MIT license.
We show the usefulness of our approach when analyzing the attention flow of general Transformers.

\emph{Related Work.}
We would like to emphasize the work on which we build: ~\citet{DBLP:conf/acl/AbnarZ20} introducing attention flow for Transformer encoders and ~\citet{DBLP:conf/acl/EthayarajhJ20} drawing a possible connection between encoder attention flows and Shapley values.
An overview on the necessity of explainability is given by~\citet{samek2017explainable} and \citet{burkart2021survey} provide a survey over existing methods.
An overview over many Shapley value formulations for machine learning models is given by~\citet{sundararajan2020many}, which are not restricted to Transformer models and not considering attention flow~\citep{lindeman1980introduction,gromping2007estimators,owen2014sobol,owen2017shapley,vstrumbelj2009explaining,vstrumbelj2014explaining,datta2016algorithmic,NIPS2017_7062,lundberg2018consistent,aas2019explaining,sun2011axiomatic,sundararajan2017axiomatic,agarwal2019new}.
Shapley values are also used for the valuation of machine learning data~\citep{ghorbani2019data}.
Visualization of raw attention values has been studied, for example, by~\citet{vig2019multiscale} and~\citet{wang2021dodrio}.
\citet{chefer2021transformer} study relevancy for Transformer
networks in computer vision by assigning local relevance based on
the Deep Taylor Decomposition principle~\cite{montavon2017explaining}.

\section{Attention Flow}
\label{sec:attn_flow}
\emph{Attention Flow}~\citep{DBLP:conf/acl/AbnarZ20} is a post-processing interpretability technique that treats the self-attention weight matrices of the Transformer encoder as a flow network and returns the maximum flow through each input token.
Formally, a flow network is defined as follows.

\begin{definition}[Flow Network]
Given a graph $G=(V,E)$, where $V$ is a set of vertices and $E \subseteq V \times V$ is a set of edges,
a flow network is a tuple $(G,c,s,t)$, where $c: E \rightarrow \mathbb{R}_\infty$ is the capacity function
and $s$ and $t$ are the source and terminal (sink) nodes respectively.
A flow is a function $f: E \rightarrow \mathbb{R}$ satisfying the following two conditions:
Flow conservation: $\forall v \in V\backslash \{s,t\}.~x_f(v) = 0,$
where $x_f:V \rightarrow \mathbb{R}$ is defined as $x_f(u) = \Sigma_{v \in V} f(v,u),$
and capacity constraint: $\forall e \in E.~ f(e) \leq c(e)$.
\end{definition}

The value of flow $|f|$ is the amount of flow from the source node $s$ to the terminal node $t$:
$
|f|=\sum _{v:(s,v)\in E}f_{sv}.
$
For a given set $K$ of nodes, we define $|f(K)|$ as the flow value from $s$ to $t$ only passing through nodes in $K$:
$
|f(K)|=\sum _{v:(s,v)\in E, v \in K}f_{sv}.
$
We define $|f_o(v)|$ to be the total outflow value of a node $v$ and $|f_i(v)|$ to be the total inflow value of a node $v$.
In optimization theory, the maximum flow problem $\mathit{max}(|f|)$~\citep{harris1955fundamentals} is to find flows that push the maximum possible flow value $|f|$ from the source node $s$ to the terminal node $t$, which we denote by $f_\mathit{max}$.


\begin{figure}[t]
\begin{subfigure}[t]{0.48\textwidth}
\centering
    \includegraphics[width=.9\textwidth]{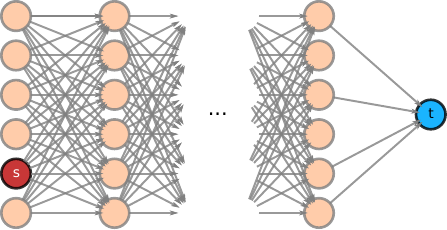}
\caption{With $|L|+1$ node columns, $|L|$ edge columns, $I'=\{i_5\}$ and $J=\{1,3,5,6\}$. The red node depicts the input token for which the maximum flow is computed. The blue node represents the terminal node $t$.}
\label{fig:encoder_attn_flow}
\end{subfigure}
\hfill
\begin{subfigure}[t]{0.48\textwidth}
    \centering
    \includegraphics[width=.9\textwidth]{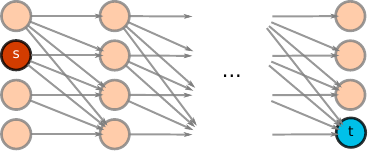}
    \caption{Input token set $O' = \{o_2\}$ and embedding $t$, where the output token $o_5$ is currently predicted. The first ``input'' token, i.e., $o_1$ is the special start token of the decoder.}
    \label{fig:decoder_attn_flow}
\end{subfigure}
\caption{Encoder attention in a flow network (left) and decoder attention in a flow network (right).}
\end{figure}

\subsection{Encoder Attention Flow}
\label{sec:encoder_only}
Given an encoder-only Transformer model, such as the BERT~\cite{DBLP:conf/naacl/DevlinCLT19} model family, with $H$ attention heads, $L$ layers, $M$ input tokens $I=\{i_1,\ldots,i_M\}$ and the resulting self-attention~tensor~$\tA^{E} \in \mathbb{R}^{H \times L \times M \times M}$.
For some $X \in \mathbb{N}$, we define $[X]$ as the set $\{1,\ldots,X\}$.
For a set of positions $J$, a subset of input tokens $I' \subseteq I$ and subset of heads $H' \subseteq H$, we construct a flow network $\mathcal{F}_{enc} (\tA^E, I',J) = (G,c,s,t)$ as follows:
%
\begin{align*}
&V := (I \times [L + 1]) \cup \{s,t\} \enspace , \phantom{scaling} c((i_j,l),v') := \begin{cases} 
      \frac{1}{H'}\sum_{h=1}^{H'} \tA^E_{h,l,k,j} & v' =(i_k,l+1)\\
      \infty & v' = t \enspace \end{cases} \enspace ,\\
&E := \{((i_j,l),(i_k,l+1))~|~i_j,i_k \in I \wedge l \in [L + 1]\} \\
& ~~~~~~~ \cup \{((i_j, L + 1), t) ~|~ i_j \in I \wedge j \in J\}\\
& ~~~~~~~ \cup \{(s,(i', 0)) ~|~ i' \in I'\} \enspace .
\end{align*}

%
%
%
We visualize this flow network translation in Figure~\ref{fig:encoder_attn_flow}.
The flow network consists of $L+1$ columns of nodes and $L$ columns of edges. 
The attention values are encoded as capacities on the edges, thus the underlying graph of the flow network requires one additional column of nodes.
Computing the maximum flow through this network determines the contribution of the set of input tokens $I'$ to the attention flow towards the final encoder embeddings given by $J$. Note that the nodes in columns greater than $1$ correspond to encoder embeddings and can not be interpreted as input tokens anymore.
Residual connections can be taken into account as proposed by~\citet{DBLP:conf/acl/AbnarZ20}: adding an identity matrix $I$ and re-normalize as $0.5\tA + 0.5I$.


By setting the start node $s$ successively to singleton sets containing only a single input token and all final embeddings to $t$, we can compute the encoder flow for every encoder input token as introduced by~\citet{DBLP:conf/acl/AbnarZ20}.

The encoder flow network construction can also be used for models including a classification task (see Section~\ref{sec:experiments}).
To determine the influence of input tokens to the attention flow towards deciding the class, the terminal node $t$ is then only connected to the final embedding of the classification token.

\subsection{Decoder Attention Flow}
\label{sec:decoder_only}
Generative Transformer models that involve a decoder require a significantly different shape of flow network.
We begin by investigating decoder-only models, with $H$ attention heads, $L$ layers, $N$ ``output'' tokens $O = \{o_1,\ldots,o_N\}$ and the self-attention tensor $\tA^D \in \mathbb{R}^{H \times L \times N \times N}$.
Since we consider decoder-only models, a prefix subset $O_\mathit{input} \subseteq O$ will be given as a problem input to the neural network model.
 Note that the first output token is always a special start token.
For a set of output tokens $O' \subseteq O$, the position $n$ of output token $o_n \in O$ and subset of heads $H' \subseteq H$, the construction of a flow network $\mathcal{F}_{dec}(\tA^D, O', n) = (G,c,s,t)$ follows the structure of the decoder self-attention:
\resizebox{\linewidth}{!}{
\begin{minipage}{\linewidth}
\begin{align*}
&V := O \times [L + 1] \enspace , \phantom{space} E := \{(o_j,l),(o_k,l+1))~|~o_j,o_k \in O \wedge l \in [L + 1] \wedge j \leq k\} \enspace ,\\
&c((o_j,l),(o_k,l+1)) := 
      \frac{1}{H'}\sum_{h=1}^{H'} \tA^D_{h,l,k,j} \enspace , 
\phantom{spa}s :=  \{(s,(o', 0)) ~|~ o' \in O'\} \enspace ,
\phantom{spac} t := (o_{n-1}, L + 1) \enspace .
\end{align*}
\end{minipage}}

We visualize the construction in Figure~\ref{fig:decoder_attn_flow}.
Because of the auto-regressive nature of the Transformer decoder, we compute the maxflow to the last embedding of the decoder as this embedding will be used in the Transformer to predict the next token.
%
%
The auto-regression, however, requires a normalization to account for the bias towards tokens that were predicted later than others (later predicted tokens have more incoming edges).
Intuitively, we require that the maxflow computation for any sub flow network $F'$ constructed from the decoder flow network $F$ to be independent of the absolute position of $F'$ in $F$.
Formally, assuming $\tA^D$ to have the same value $c$ for every entry, i.e., the capacity of every edge in the resulting-flow network is fixed to $c$, we require for every position $n$ that 
$
\forall o_m \in O.~\mathit{maxflow}(\mathcal{F}_{dec}(\tA^D, \{o_m\}, n)) = c,
$
which we call \emph{positional independence}.
We ensure this by dividing the result of a max flow computation for a given start token $o_m$ and end token $o_n$ by
$
1 + ( O - (n - m) ) - m.
$
%
For a subset $O' \subseteq O$ and a position $n$ (where $\forall o'_m \in O'. m < n)$ and heads $H'$, we can thus compute the influence of the token set $O'$ to the total attention flow towards the embedding that predicts the $n$-th token whether it served as part of the problem input or is an already predicted output token.


\subsection{Encoder-Decoder Attention Flow}
For Transformer models consisting of an encoder and a decoder, we combine both flow network translations with the encoder-decoder attention.
Figure~\ref{fig:enc_dec_attn_flow} shows the structure of the flow network for a Transformer model with an encoder (top) and a decoder (bottom). The last nodes of the flow network corresponding to the final embedding of the encoder are, following the Transformer architecture, connected to every node layer of the network corresponding to the decoder. We omit some encoder-decoder edges for better visualization.
Given a Transformer with $H$ attention heads, $L$ layers, $M$ input tokens $I = \{i_0,\ldots,i_M\}$, $N$ output tokens $O = \{o_0,\ldots,o_N\}$, and resulting encoder self-attention tensor $\tA^E \in \mathbb{R}^{H \times L \times M \times M}$, decoder self-attention tensor $\tA^D \in \mathbb{R}^{H \times L \times N \times N}$ and encoder-decoder attention tensor $\tA^C \in \mathbb{R}^{H \times L \times N \times M}$. For a set of input tokens $I'$, the position $n$ of output token $o_n$ and subset of heads $H' \subseteq H$, we construct a flow network $\mathcal{F}(\tA^E, \tA^D, \tA^C, I', n) = (G, c, s, t)$ from flow networks $\mathcal{F}_{enc}(\tA^E, I',\emptyset) = ((V_{enc}, E_{enc}),c_{enc}, s_{enc}, t_{enc})$ and $\mathcal{F}_{dec}(\tA^D, \emptyset, n) = ((V_{dec}, E_{dec}),c_{dec}, s_{dec}, t_{dec})$ as follows:
\begin{align*}
&V := V_{enc} \cup V_{dec} \cup s \enspace , \phantom{space} E := E_{enc} \cup E_{dec} \cup \{((i_j, L+1), v)~|~i_j\in I \wedge v \in V_{dec}\} \enspace ,\\
&\phantom{a lot of space, even more though still} ~~~~~~~ \cup \{(s_{enc},(o_m,0)) ~|~ o_m \in I \wedge m < n\}\\
&c(v,v') := \begin{cases}
        c_{enc}(v, v') & v = (i_j, l),v' = (i_k, l'),\\
                       &     \phantom{spaceing}i_j, i_k \in I\\
        c_{dec}(v, v') & v = (o_j, l), v' = (o_k, l'), \\
        & \phantom{spaceing}o_j, o_k \in O\\ 
        \frac{1}{H'}\sum_{h=1}^{H'} \tA^C_{h,l,k,j} &v = (i_j, L + 1),\\
                &~v' = (o_k, l), i_j \in I, o_k \in O\\
        \infty & v = s, v' \in I' \enspace ,\\
      \end{cases}\\
&t := (o_n, L+1) \enspace ,
\end{align*}

where $l_e$ denotes a layer from the encoder and $v_d$ denotes a node from the decoder.
Again, we have to normalize to account for the auto-regressive bias, i.e., require positional independence.
For a given set on input tokens $I'$ and heads $H'$, we can thus asses the contribution of this set to the total attention flow towards the embedding that predicts the $n$-th token by computing the maxflow through this network.
If one is interested in the influence of an already computed output token $o_m$, where $m<n$, on the prediction of $o_n$, then the construction for the decoder-only case in Section~\ref{sec:decoder_only} applies.

\subsection{Subset and Single Head Attention Flow}
\label{sec:single_head}
The flow network constructions are applicable to subsets of heads, and especially, single heads.
Because of the particularities of many Transformer implementations, the results have to be interpreted carefully.
The results of head computations are joined using a linear projection, such that each head has access to the computations of all heads in the previous layers.
Meaning that, the task of a head in layer $l$ can be independent of its task in previous layers $l' < l$.

In practice, however, heads are biased towards keeping their respective tasks, such that we also found good interpretability results by considering the attention flow of attention heads independently (see Section~\ref{sec:experiments}).
A flow network can be constructed for a single head by following the constructions above setting $H'$ to every singleton set.
As mentioned, however, these results have to be analyzed carefully with the particularities of Transformer implementations in mind.

\subsection{Optimizations}
In this section, we give a brief overview on possible relaxations on the flow network if the computation time of the maxflow for large Transformer models exceeds time limits.
First, note that the flow network only needs to be constructed once.
%
As expected, the computation time of the maxflow in the network constructions increases with larger input and output sequences.
Running time can be traded against heuristically shrinking the size of the flow network.
This can be done in two dimensions.
Following the practical assumption that heads often keep their tasks throughout subsequent layers, the first dimension is to shrink the flow network on the $x$-axis.
This can be done by simply skipping some of the inner layers of the network or by merging layers by taking the average of the raw attention values across layers as capacities.
Furthermore, the network can also be shrunk in the $y$-dimension.
This can be done in the same way by grouping input and output tokens. For example, tokens $\mathit{predict}$ and $\mathit{ed}$ can be combined into one node.
\label{sec:optimization}

\begin{figure}[t]
    \centering
    \includegraphics[scale=1.5]{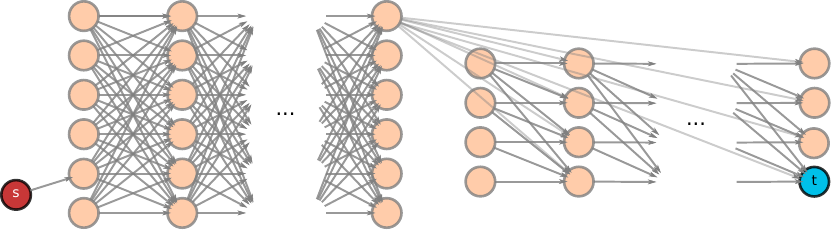}
    \caption{Sketch of the Encoder-decoder attention flow network for input token $i_5$ and embedding $t$, which is used to predict $o_5$. Encoder-decoder connections are sketched for the first node.} 
    \label{fig:enc_dec_attn_flow}
\end{figure}

\section{Shapley Value Explanations}
In this section, we show how the extended flow network constructions over the Transformer decoder $\mathcal{F}_{dec}(\tA^D, O', n)$ and $\mathcal{F}(\tA^E, \tA^D, \tA^C, I', n)$ induce Shapley value explanations for the tokens of the input sequence.
\label{sec:shapley}
The Shapley value~\citep{shapley201617} is a solution concept determining the impact of a player in cooperative game theory and an increasingly popular concept to determine the influence of certain input features on a model's decision.
\begin{definition}
A game with transferable utility (TU) is a pair $(P,v)$, with $P = \{1,\ldots,p\}$ being a finite
set of players and $v: 2^P \rightarrow \mathbb{R}$ being the payoff function.
\end{definition}

A subset $S \subseteq P$ is called a \emph{coalition}.
The payoff function $v$ assigns every coalition of players $S$
a real number $v(S) \in \mathbb{R}$ with $v(\emptyset) = 0$.
We denote the \emph{share} of a player $i$ of the allocated payoff as $\varphi_i(v)$.
The encoding of the attention values as a flow network can be seen as a TU game.
A node in the flow network represents a player and the total flow through the network represents the total payoff~\citep{DBLP:conf/acl/EthayarajhJ20}.
The Shapley values of the players in a TU game are formally defined as follows.

\begin{definition}[Shapley Value]
Let $\Pi(P)$ be the set of all player permutations and let $\pi \in \Pi(P)$ be a permutation of players.
Let all players ahead of a player $i$ be defined as $P_{<i}(\pi) := \{j \in P: \pi(j) < \pi(i)\}.$
The Shapley value $\varphi$ is defined as the share of payoff for a given player $i \in P$:
$\varphi_i(P,v) := \frac{1}{p!} \sum_{\pi \in \Pi(P)} (v(P_{<i}(\pi) \cup \{i\}) - v(P_{<i}(\pi)))$.
\end{definition}

From a game-theoretic viewpoint, Shapley values are well-suited for determining the payoff share that players deserve, as the values satisfy the following desirable properties.

\begin{property}[Efficiency]
All of the available payoff $v(P)$ is distributed between the players:
$v(P) = \sum_{i \in P} \varphi_i(P,v).$
\end{property}

\begin{property}[Symmetry]
Two players that have the same impact on the total payoff when joining a coalition,
receive the same share of the payoff: $\forall S \subseteq P\backslash\{i,j\}.~v(S \cup \{i\}) = v(S \cup \{j\})\rightarrow \varphi_i(P,v) = \varphi_j(P,v).$
\end{property}

\begin{property}[Null Player]\label{def:null:player}
A player that has zero impact upon joining a coalition, receives no share of the total payoff:
$\forall S \subseteq P\backslash \{i\}.~v(S) = v(S\cup\{i\})\rightarrow
\varphi_i(P,v) = 0.$
\end{property}

\begin{property}[Additivity]
The share of a player in TU game $(P,v+w)$ is the sum of their shares
in games $(P,v)$ and $(P,w)$:
$\forall i \in P.~\varphi_i(P,v+w) = \varphi_i(P,v) + \varphi_i(P,w).$
\end{property}

The properties above are also responsible for making Shapley values an attractive approach for explaining a model's decisions in the ML community, i.e., features that do not contribute to the accuracy of a model should be null players, and features that contribute equally should satisfy symmetry.

\begin{proposition}[Decoder-Only Flow Is a Shapley Value]
\label{prop:decoder_only}
Consider a Transformer decoder with $H$ attention heads, $L$ layers, $N$ ``output'' tokens  $O = \{o_1,\ldots,o_{N}\}$ and the self-attention tensor $\tA^D \in \mathbb{R}^{H \times L \times N \times N}$.
Let $f^o_\mathit{max}$ be the maxflow computed in the flow network $\mathcal{F}_{dec}(\tA^D, \{o\}, n)$ as defined in the previous section.
Consider the TU-game $(P, v)$, where the players $p \in P = \{1, ..., N\}$ correspond to nodes $(o_p,0)$ from the first layer of the Transformer decoder.
For a given coalition $S \subseteq P$, let the value function be $v(S) = \sum_{s \in S} f^{o_s}_{max}$, i.e., the sum of max-flows of nodes corresponding to $S$. Then, the max-flow $f^{o_s}_{max}$ for some $p \in P$ is its Shapley Value.
\end{proposition}
\begin{proof}
The proof immediately follows from the fact that every max-flow of some node $f^{o}_{max}$ is an independent computation and the payoff of a coalition is defined as a sum of these independent contributions, which trivially qualifies as a Shapley value.
\end{proof}
Although this theoretical correspondence to a Shapley value is trivial, we show in our experiments in the following section that the maxflow computation indeed yields meaningful explanations for the network's attention flow.

Note that our line of reasoning significantly differs from \citet{DBLP:conf/acl/EthayarajhJ20}. In particular, we compute a seperate max-flow for every token in the set of players. This is because key assumptions about flow networks that they make in their proof do not hold: They argue that as long as nodes come from the same layer, blocking flow through some of these nodes does not change the possible flow through the others, such that they can deduce that the utility a player adds when joining a coalition is independent of the identity of the players already in the coalition. However, this is clearly not the case: Several nodes from the same layer can compete for capacity downstream in the network even if they have no direct connection, e.g., if we have two tokens $o_1, o_2$ in one layer each attended to with 0.5 attention by a node $o_3$ which itself is only attended to with $0.5$ attention. Now, clearly the utility $o_1$ adds upon joining a coalition as defined by \citet{DBLP:conf/acl/EthayarajhJ20} does in fact depend on whether $o_2$ is already part of it. While they unfortunately have not formalized their Attention Flow construction explicitly, we deduce from the above discussion that it may in fact violate the Symmetry of a Shapley value, as the payoff for $o_1$ and $o_2$ can be unequally allocated.


The ideas outlined for Proposition~\ref{prop:decoder_only} also apply to the encoder-decoder attention flow. In the following, let $f^i_\mathit{max}$ be the maxflow computed in the flow network construction $\mathcal{F}(\tA^E, \tA^D, \tA^C, \{ i \}, n)$ over the Transformer with $H$ attention heads, $L$ layers, $M$ input tokens $I = \{i_0,\ldots,i_M\}$, $N$ output tokens $O = \{o_0,\ldots,o_N\}$, and resulting encoder self-attention tensor $\tA^E \in \mathbb{R}^{H \times L \times M \times M}$, decoder self-attention tensor $\tA^D \in \mathbb{R}^{H \times L \times N \times N}$ and encoder-decoder attention tensor $\tA^C \in \mathbb{R}^{H \times L \times N \times M}$.


\setcounter{corollary}{1}

\begin{corollary}[Encoder-Decoder Flow Is a Shapley Value]
\label{prop:enc_dec}
Consider the TU-game $(P, v)$, where the players $p \in P = \{1, ..., N\}$ correspond to nodes $(i_p,0)$ from the first layer.
Let the value function or a given coalition $S \subseteq P$ be defined as $v(S) = \sum_{s \in S} f^{i_s}_{max}$, i.e., the sum of max-flows of nodes corresponding to $S$. Then, the max-flow $f^{i_s}_{max}$ for some $p \in P$ is its Shapley Value.
\end{corollary}

\section{Experiments}
\label{sec:experiments}
In this section, we exemplarily report on experiments conducted in the domains of natural language processing (NLP) and logical reasoning (LR).
We implemented the constructions from Section~\ref{sec:attn_flow}.\footnote{The code and experiments are part of the library ML2 (\url{https://github.com/reactive-systems/ml2}).}
The architecture details of the models used in the experiments are shown in Table~\ref{tab:networks}.
We visualize the maxflow attention values in a heatmaps (see, for example, Figure~\ref{fig:propsat:null:player}).
The maxflow is computed with \textsc{networkx}~\cite{networkx} and the heatmaps comparing the attention flow from input/predicted token to current predicted token are visualized with \textsc{seaborn}~\cite{Waskom2021}.
The heatmaps are either showing only the attention flow from input tokens if the model is encoder-only (\textit{enc.}), are seperated into different heatmaps for input tokens and auto-regressive tokens for encoder and decoder (\textit{enc.} + \textit{dec.}), or show one heatmap for all tokens if the architecture is decoder only (\textit{dec.}).
Higher values represent higher attention. 
%
%
%
%
\begin{figure}[t]
\begin{subfigure}[b]{.62\textwidth}
    \centering
    \resizebox{.95\textwidth}{!}{
    \includegraphics[width=.32\linewidth, valign=t]{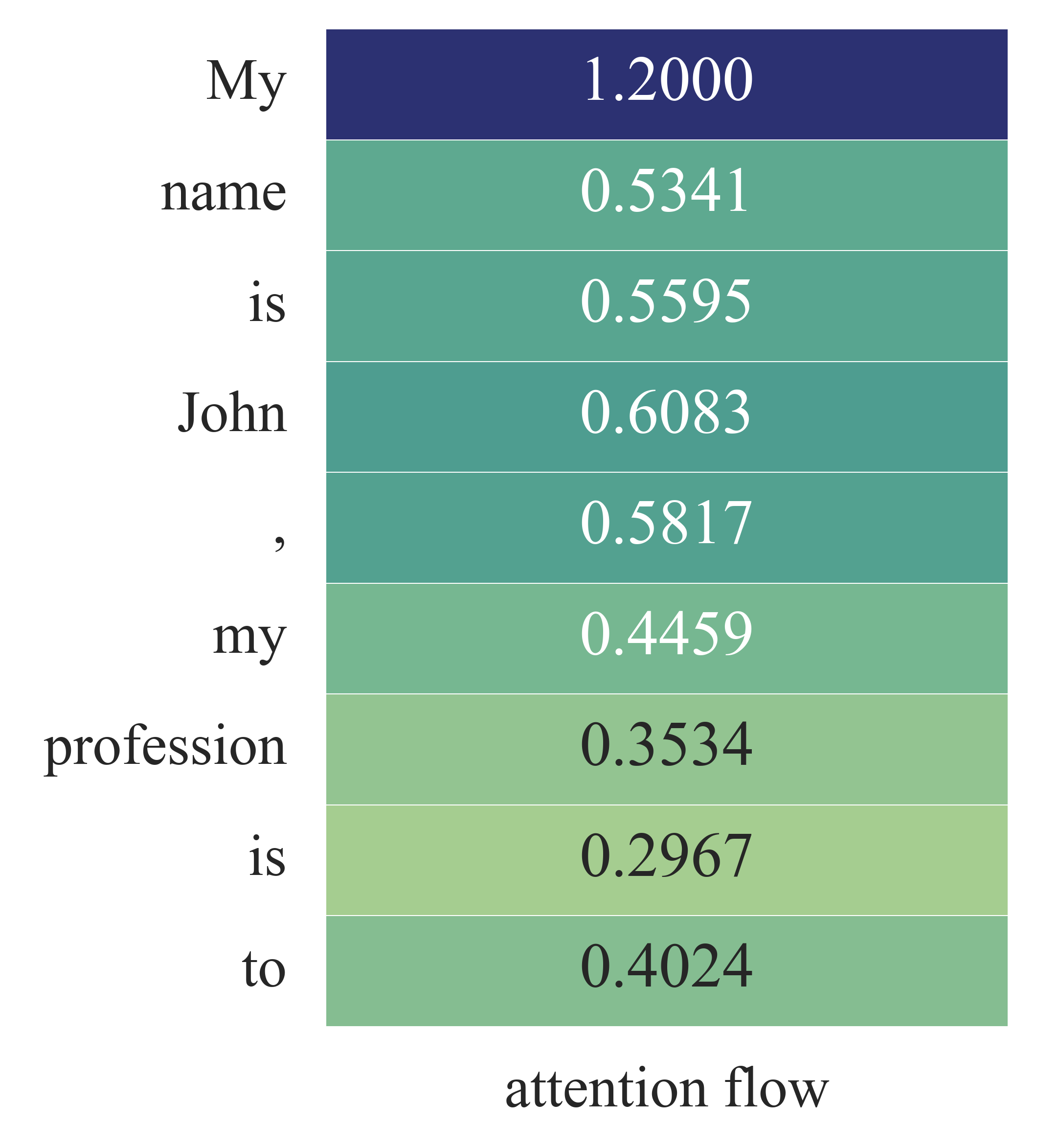}
    \includegraphics[width=.32\linewidth, valign=t]{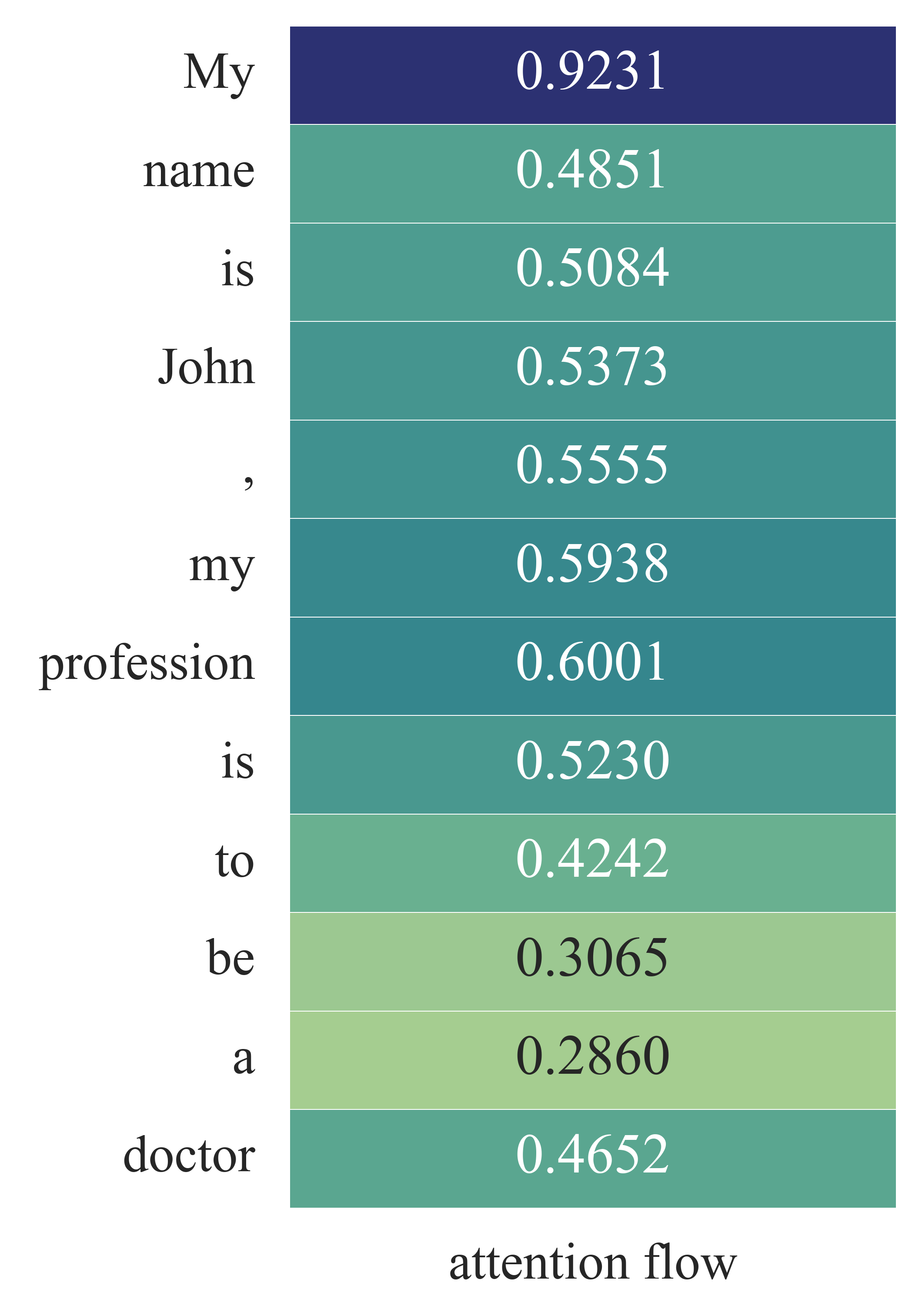}
    \includegraphics[width=.32\linewidth, valign=t]{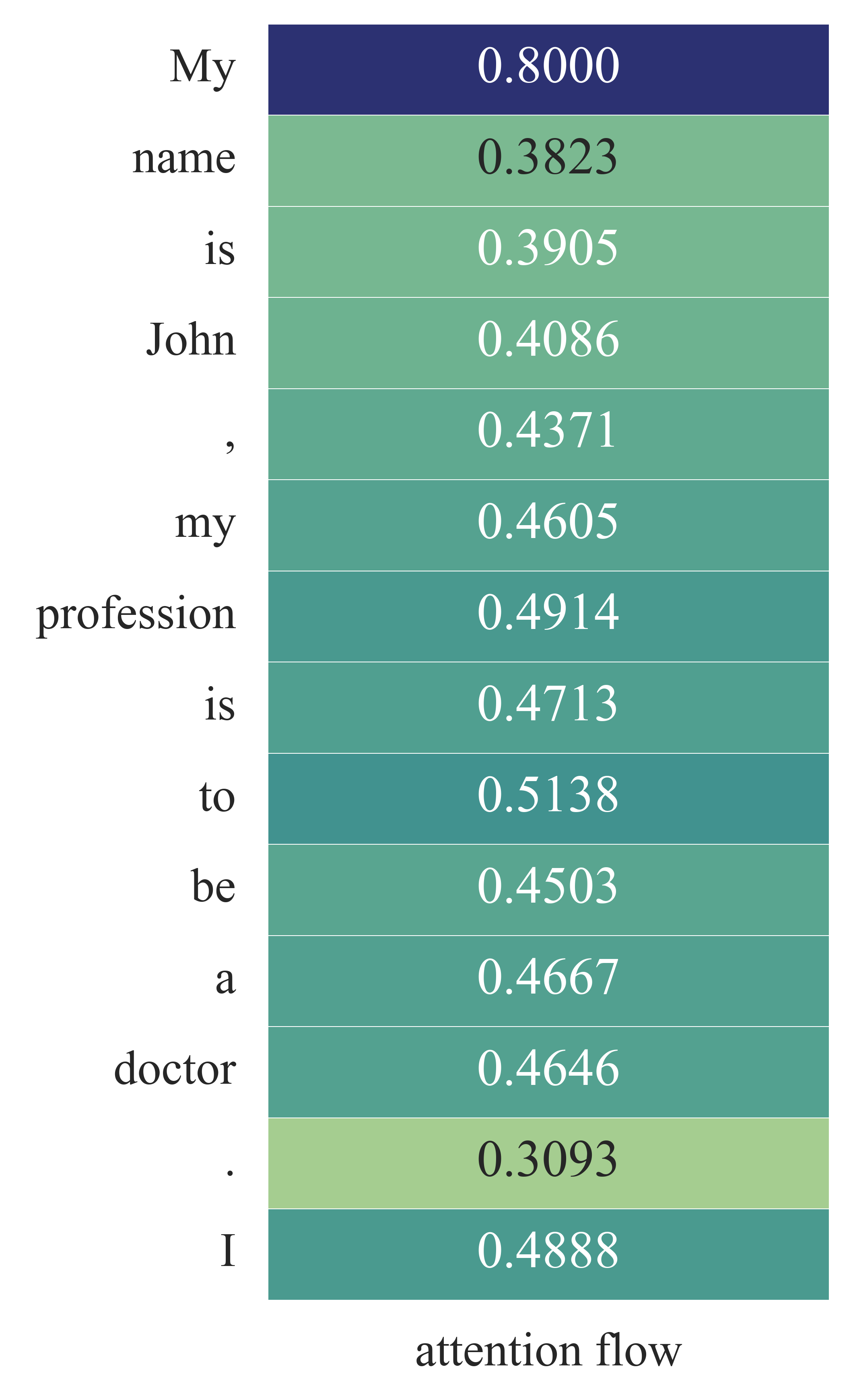}
    }
    \caption{}
    \label{fig:GPT2:flow:changes}
\end{subfigure}
\hfill
\begin{subfigure}[b]{.37\textwidth}
 \centering
    \includegraphics[width=.49\linewidth,valign=t]{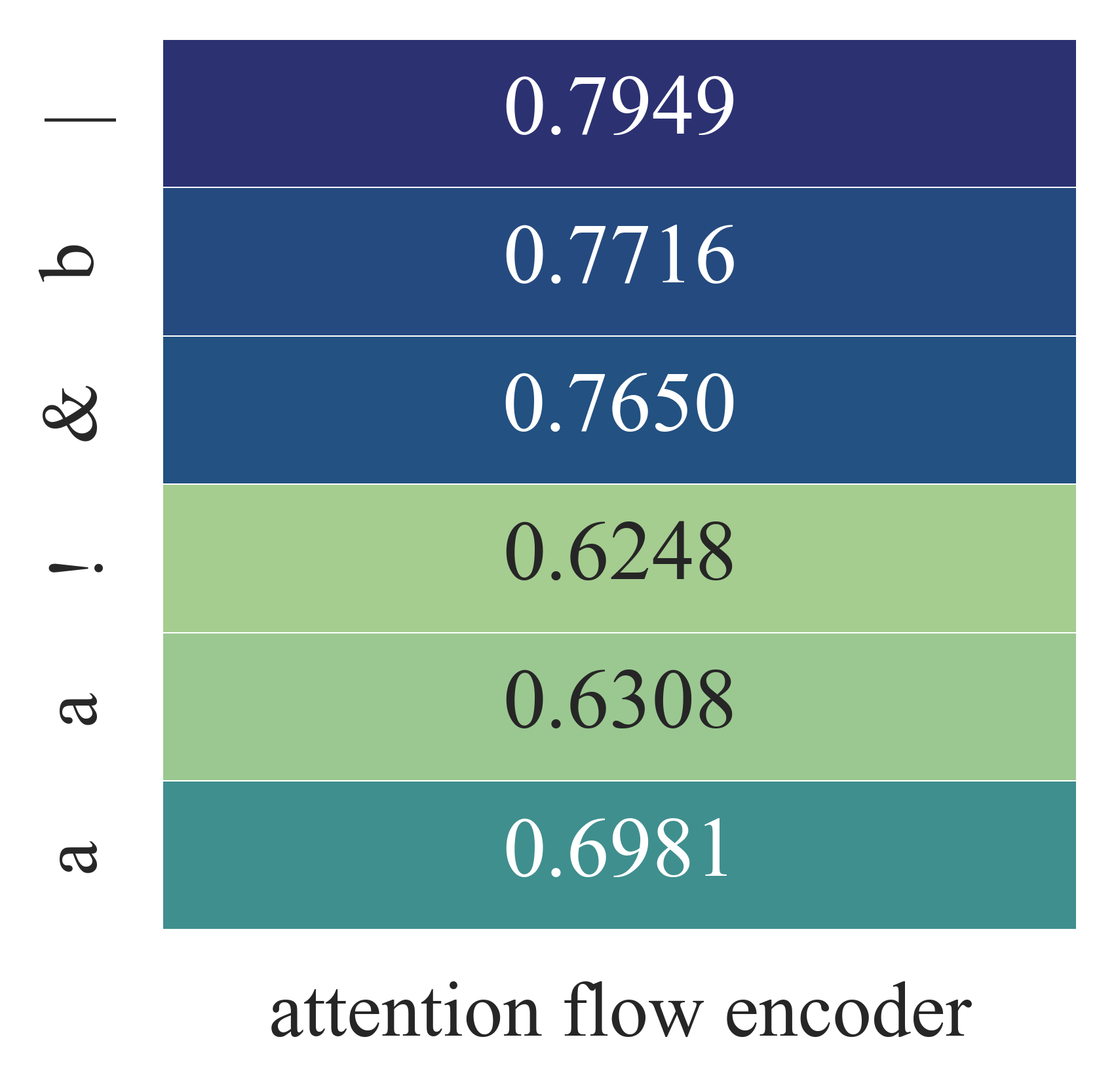}
    \includegraphics[width=.49\linewidth,valign=t]{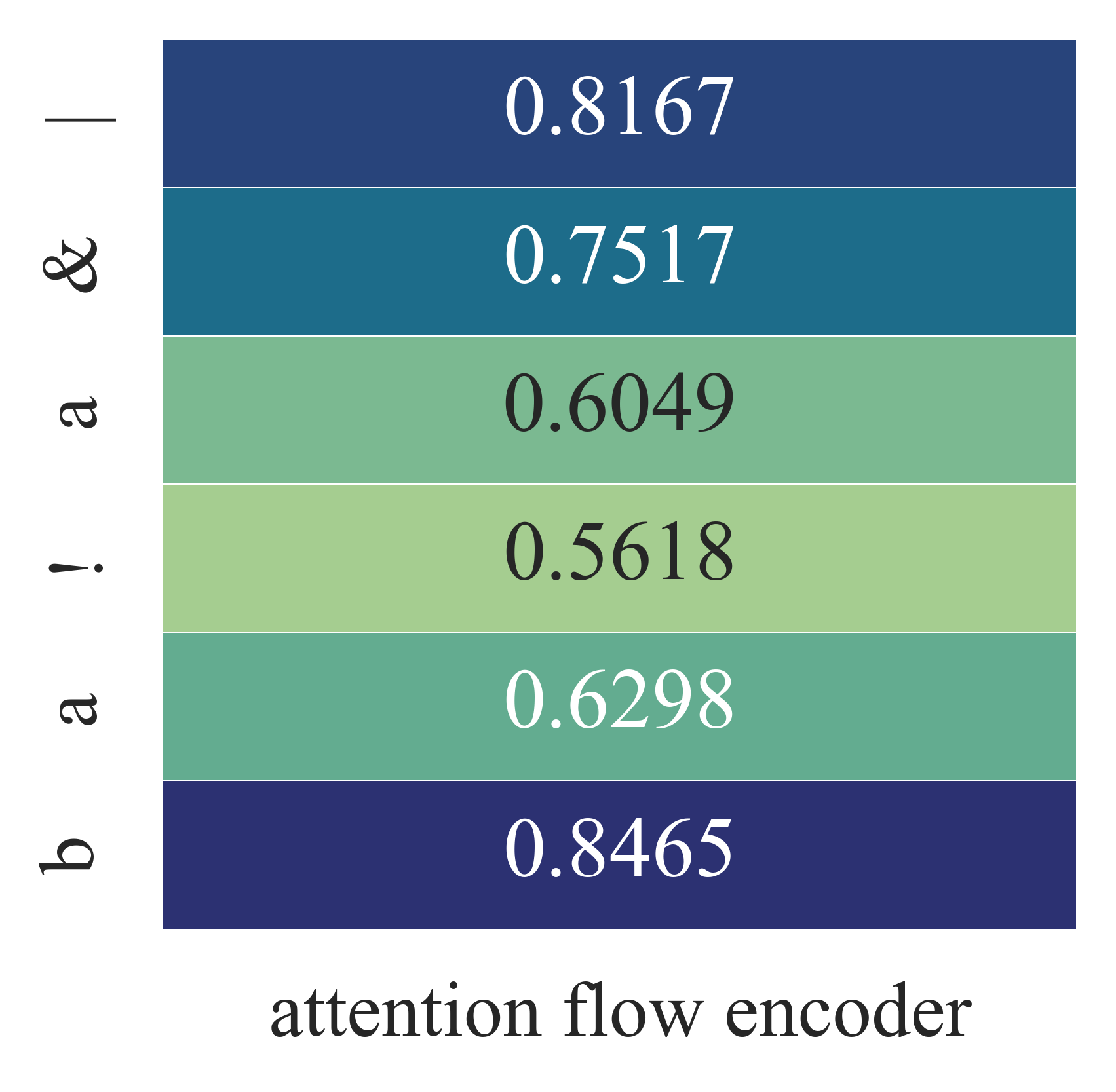}
    \includegraphics[width=.49\linewidth,valign=t]{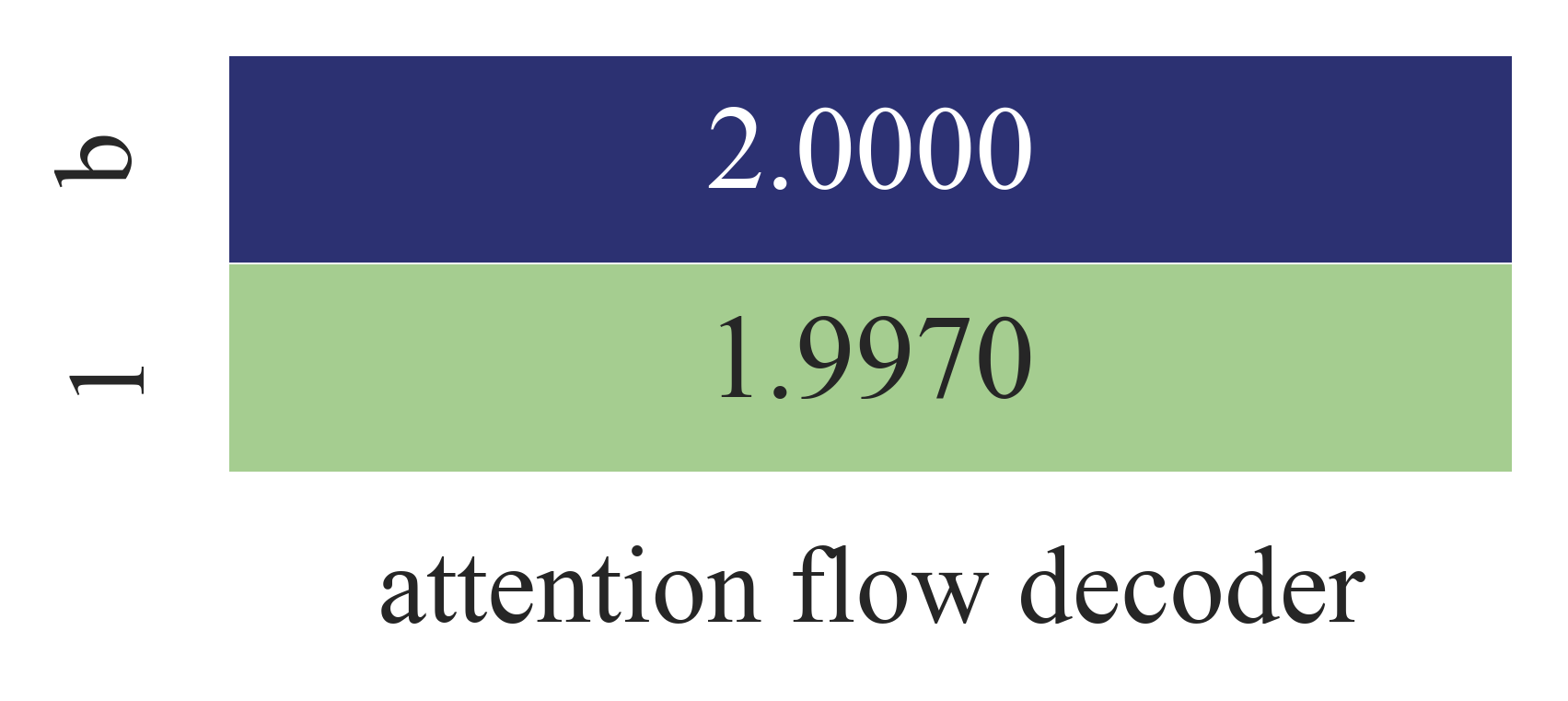}
    \includegraphics[width=.49\linewidth,valign=t]{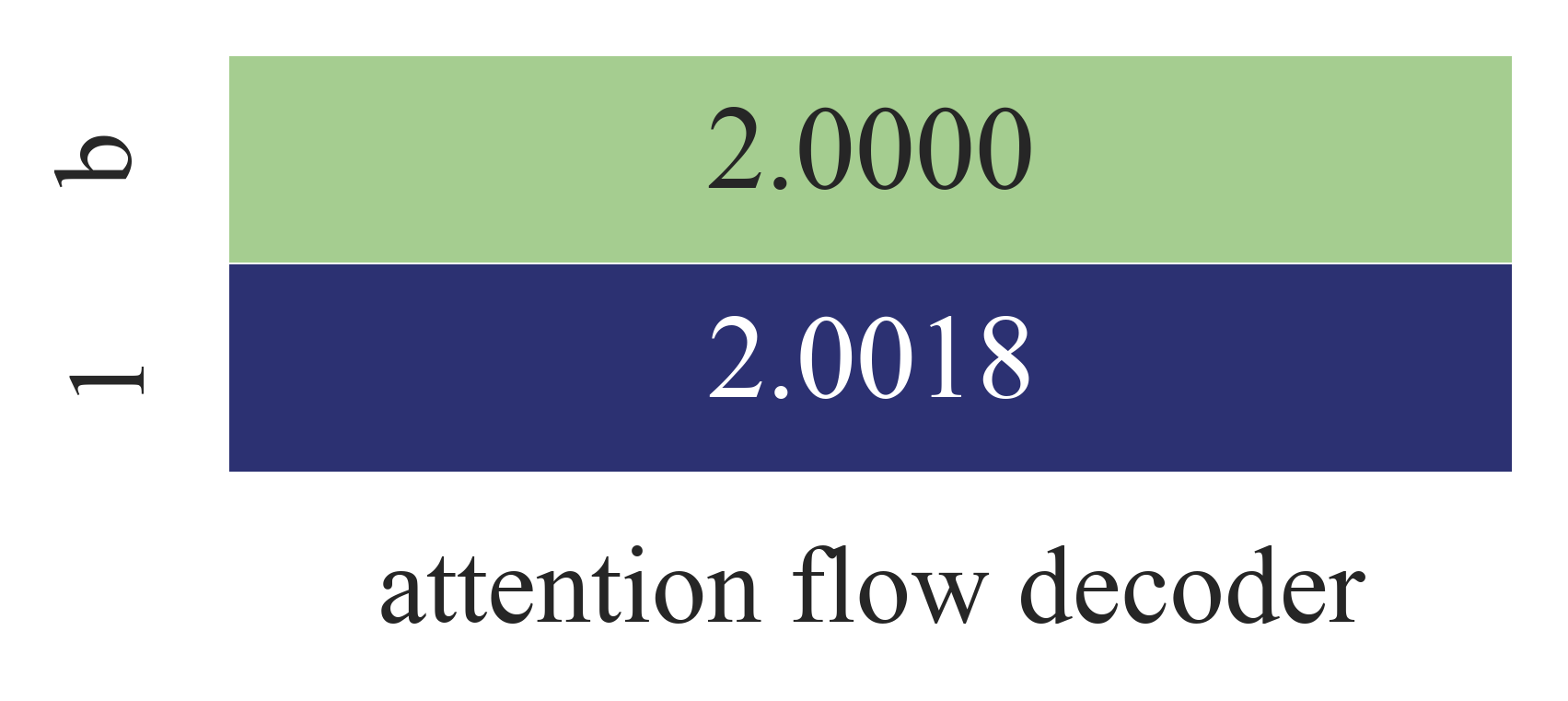}
    \caption{}
    \label{fig:propsat:null:player}
\end{subfigure}
\caption{Heatmap of the attention flow of the GPT-2 model after $1$, $4$, and $6$ predicted tokens in (a) and Heatmap depicting the attention flow for unimportant token detection in SAT assignments, reading from left to right in (b).}
\end{figure}
\subsection{Token Relevancy in Decoder Flow}
\label{sec:null_detection}
\textbf{Text completion.}
We demonstrate that our technique captures the attention flow changes during auto-regressive decoding.
For this experiment, we track the attention flow changes in GPT-2~\cite{radford2019language} while decoding the predicted tokens.
The input sequence to this decoder-only mode was ``My name is John, my profession is''. Figure~\ref{fig:GPT2:flow:changes} depicts the attention flow after decoding the first, fourth and sixth token.
The resulting flow network can be found in Figure~\ref{fig:decodernetwork} in the appendix.
Note that, in general, GPT-2 models attends the first token the most (cf. Section~\ref{sec:bias:detection}).
The differences in the attention flow are clearly visible as the attention flow on previous tokens is different for each decoding step.
Most notably, the attention flows heavily shifts towards the token ``profession'' when predicting the token ``doctor''.
We observed these heavy switches in decoder attention flow values throughout our experiments, which is why this approach is a valuable addition to existing analysis method.
The computation time of the flow values for this example took $1.38$, $1.50$, and $2.09$ seconds respectively on a personal machine.

\begin{table}[b]
\begin{subtable}[b]{0.46\textwidth}
\centering
\resizebox{\textwidth}{!}{
\begin{tabular}{cccc}
    \toprule
     Input & Negative & Neutral & Positive \\
     \midrule
     John is a killer.  &  0.9548 & 0.0417 & 0.0034\\
     John is a good killer.  & 0.8949 & 0.0967 &  0.0084 \\
     John is a good killer \includegraphics[height=1.4\fontcharht\font`\B, valign=c]{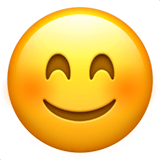} & 0.0981 & 0.3166 & 0.5853\\
     \bottomrule
\end{tabular}
}
\caption{}
\label{fig:Johnisakiller}
\end{subtable}
\begin{subtable}[b]{0.53\textwidth}
\resizebox{\textwidth}{!}{
\begin{tabular}{cccc}
    \toprule
     Network & Heads & Layers & Architecture \\
     \midrule
     DialogPT-medium (MIT)  &  16 & 24 & \textit{dec.}\\
     OPUS-MT-EN-DE (MIT)  & 6 & 8, 8 &  \textit{enc.} + \textit{dec.} \\
     PropSat (MIT) & 4 & 4, 4 & \textit{enc.} + \textit{dec.} \\
     LTLSat (MIT) & 4 & 4,4 & \textit{enc.}+ \textit{dec.} \\
     GPT-2 (MIT) & 12 & 12 & \textit{dec.}\\
     RoBERTa (MIT) & 12 & 12 & \textit{enc.}\\
     \bottomrule
\end{tabular}
}
     \caption{}
    \label{tab:networks}
\end{subtable}
\caption{Results of the sentiment analysis in (a) and the parameter overview of the models in (b).}
\label{tab:table_1}
\end{table}

\textbf{Satisfying assignments for SAT.}
In this experiment, we considered the problem of computing a satisfying assignment to a propositional logical formula.
A formula in propositional logic is constructed out of variables and Boolean connectives $\neg$ (not), $\vee$ (or), $\wedge$ (and), $\rightarrow$ (implication), and $\leftrightarrow$ (equivalence).
For example, let the following propositional formula be given:
$
b \vee (a \wedge \neg a)
$.
A satisfying assignment is a mapping from variables to truth values, such that the formula evaluates to true.
For example, a satisfying assignment for the formula above is the mapping $\{b \mapsto 1, a \mapsto 0\}$.
The variable $a$, however, has no impact on the truth value of the formula. As long as $b$ is set to $1$, a can be predicted either as $1$ or $0$.
We conducted an experiment to detect parts of propositional formula that have no impact on predicted assignments.
%
%
We trained a Transformer with an encoder and decoder to predict satisfying assignments similar to~\cite{DBLP:conf/iclr/HahnSKRF21}.
The attention flow values for the following two propositional formulas are depicted in Figure~\ref{fig:propsat:null:player}:
$\mathit{PropSAT}_1 := b \vee (a \wedge \neg a) ~\text{in tokens:}~ b|(a \& !a)$ and
$\mathit{PropSAT}_2 := (a \wedge \neg a) \vee b ~\text{in tokens:}~(a \& !a) | b$.
%
In both formulas, the disjunct $(a \wedge \neg a)$ plays no role for the satisfying assignment since any mapping of $a$ will evaluate this subformula to false.
Regardless of the position in the formula, the flow computation of the network detects this as unimportant: the inputs to the encoder $a$ and $\neg a$ have significantly less influence to the total attention flow than $b$.



\subsection{Bias Detection}\label{sec:bias:detection}

\textbf{Sentiment detection.} The flow analysis can be used to detect biases in the Transformer models.
In this experiment, we used RoBERTa~\cite{DBLP:journals/corr/abs-1907-11692} finetuned for sentiment analysis on the TweetEval~\cite{DBLP:conf/emnlp/BarbieriCAN20} benchmark and computed the influence of input tokens to the total flow deciding the classification.
We computed the attention flow values for the input tokens of the following sentences and their results, shown in Tab.~\ref{fig:Johnisakiller}
The resulting flow network can be found in Figure~\ref{fig:encodernetwork} in the appendix.
While the first two sentences ``John is a killer.'' and ``John is a good killer.'' are correctly labeled with a negative sentiment (even when having the adjective ``good'' in the sentence), having an emoji in the sentence immediately shifts the sentiment to be (falsely) labeled as positive.
The computation of the attention flow is visualized in Figure~\ref{fig:John} in the appendix.
For the first two sentences, the attention on \textit{killer} is the highest, considering only non-special tokens.
Although the same holds for the third sentence, i.e., the attention flow denotes killer as the most important word, the low attended smiley changes the sentiment to \textit{positive}.
When computing the attention flow for each head, we indeed observe heads with an attention flow of $1.0$ to the emoji (see Figure~\ref{fig:DialogPT:bias} in the appendix).
\begin{figure}
    \centering
    \begin{subfigure}[b]{.57\textwidth}
    \includegraphics[width=.32\linewidth, valign=t]{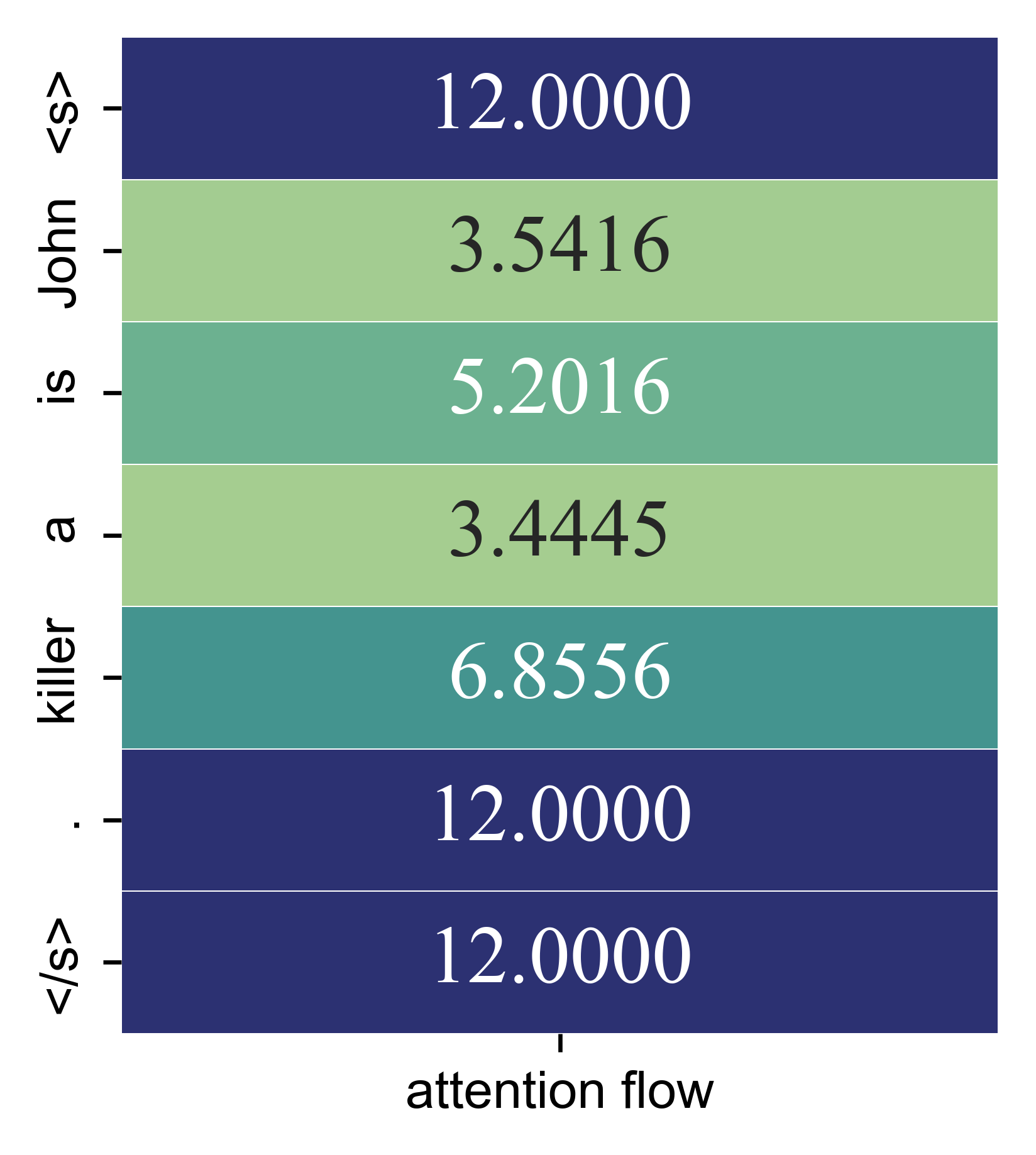}
    \includegraphics[width=.32\linewidth, valign=t]{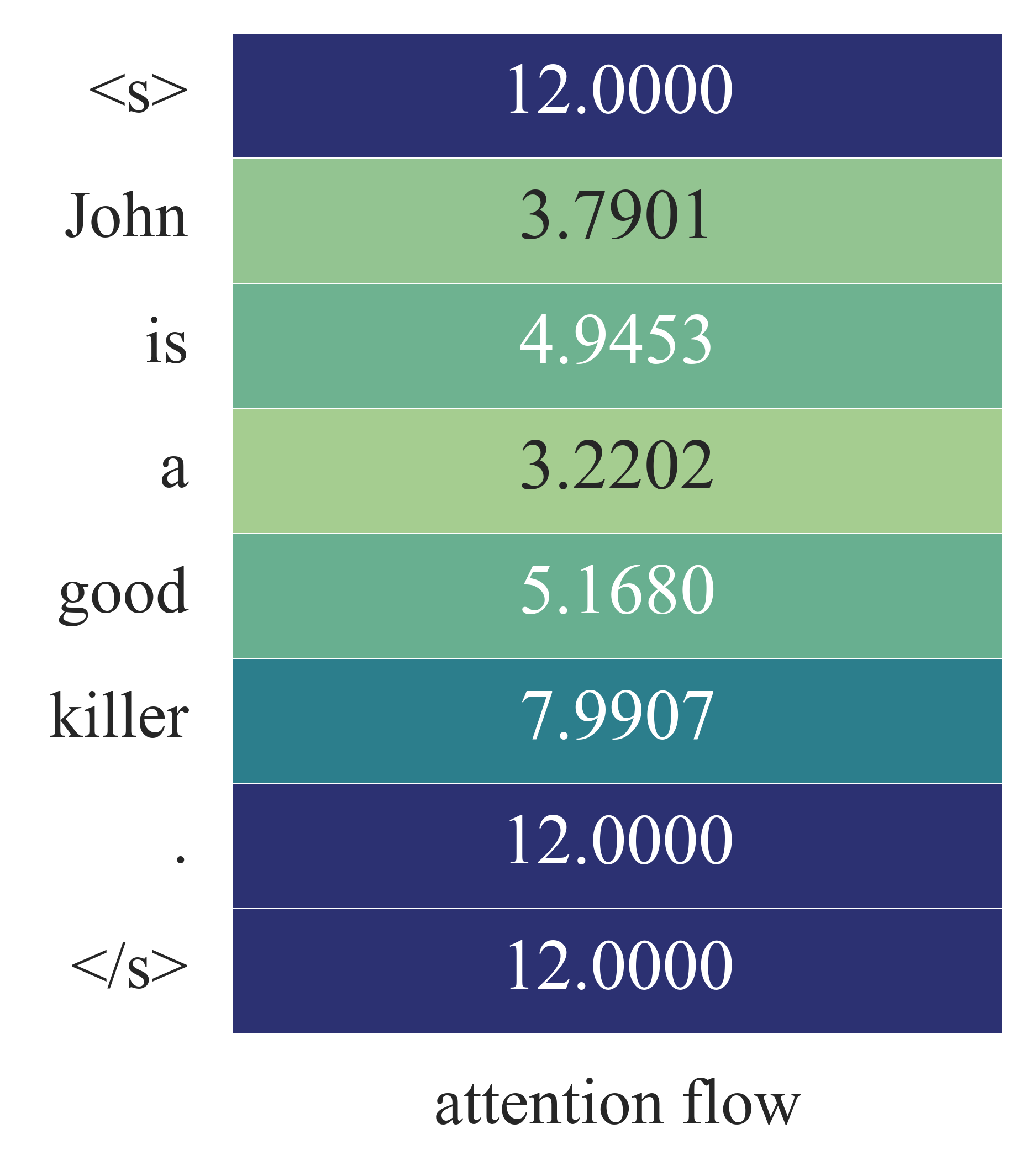}
    \includegraphics[width=.32\linewidth, valign=t]{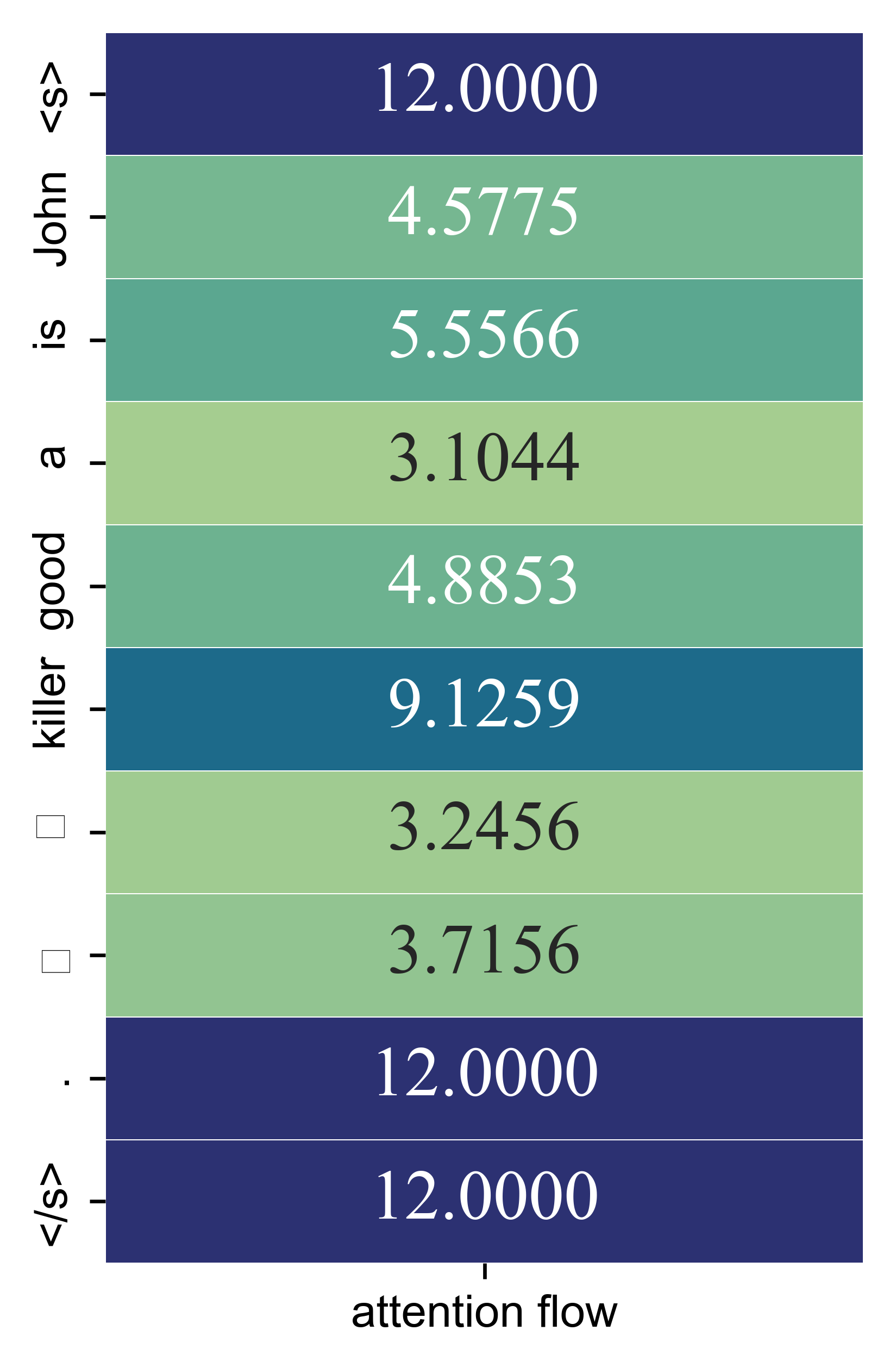}
    \caption{}
    \label{fig:John}
    \end{subfigure}
    \hfill
    \begin{subfigure}[b]{.42\textwidth}
    \centering
    \includegraphics[width=.4\linewidth,valign=t]{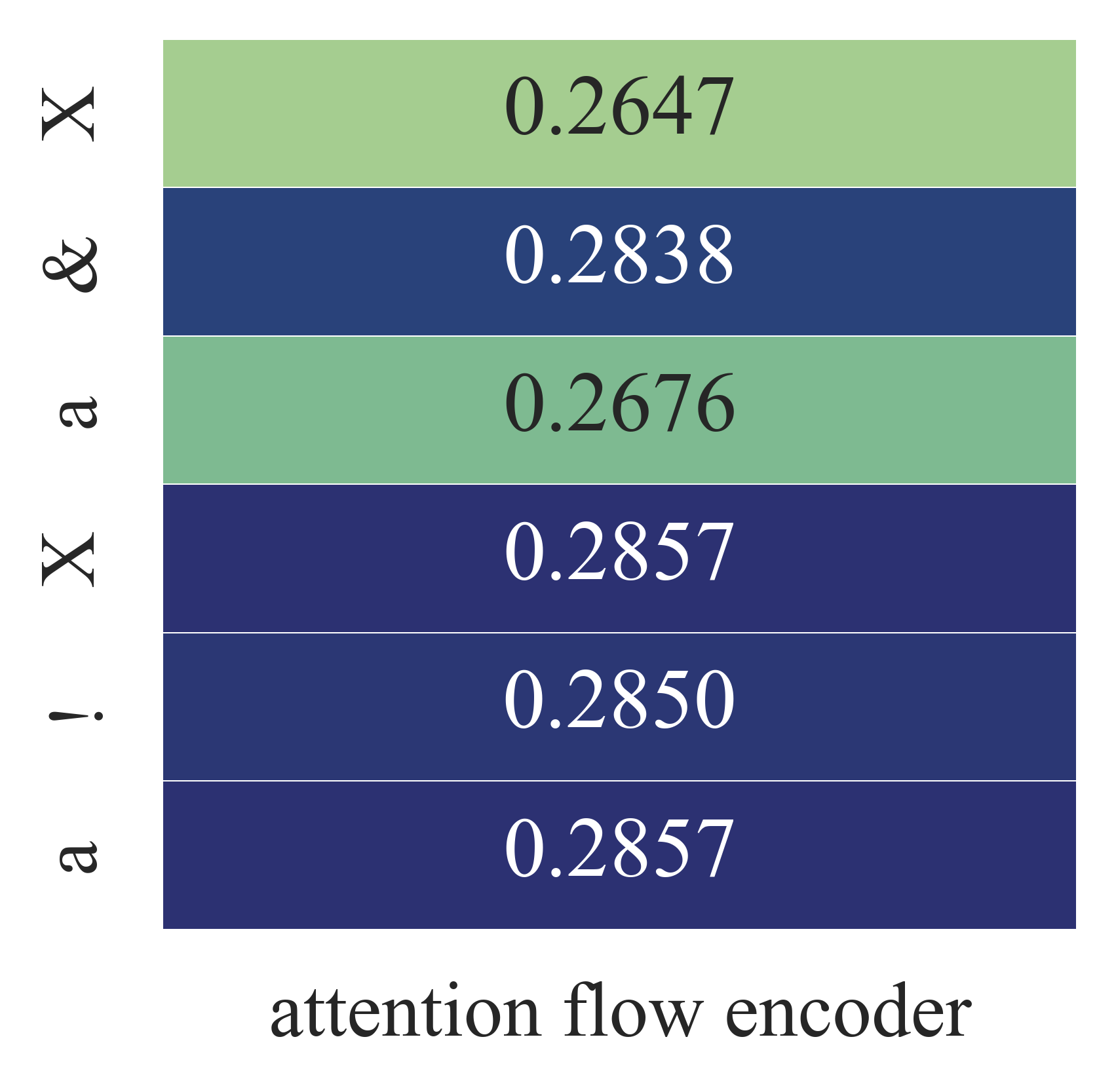}
    \includegraphics[width=.4\linewidth,valign=t]{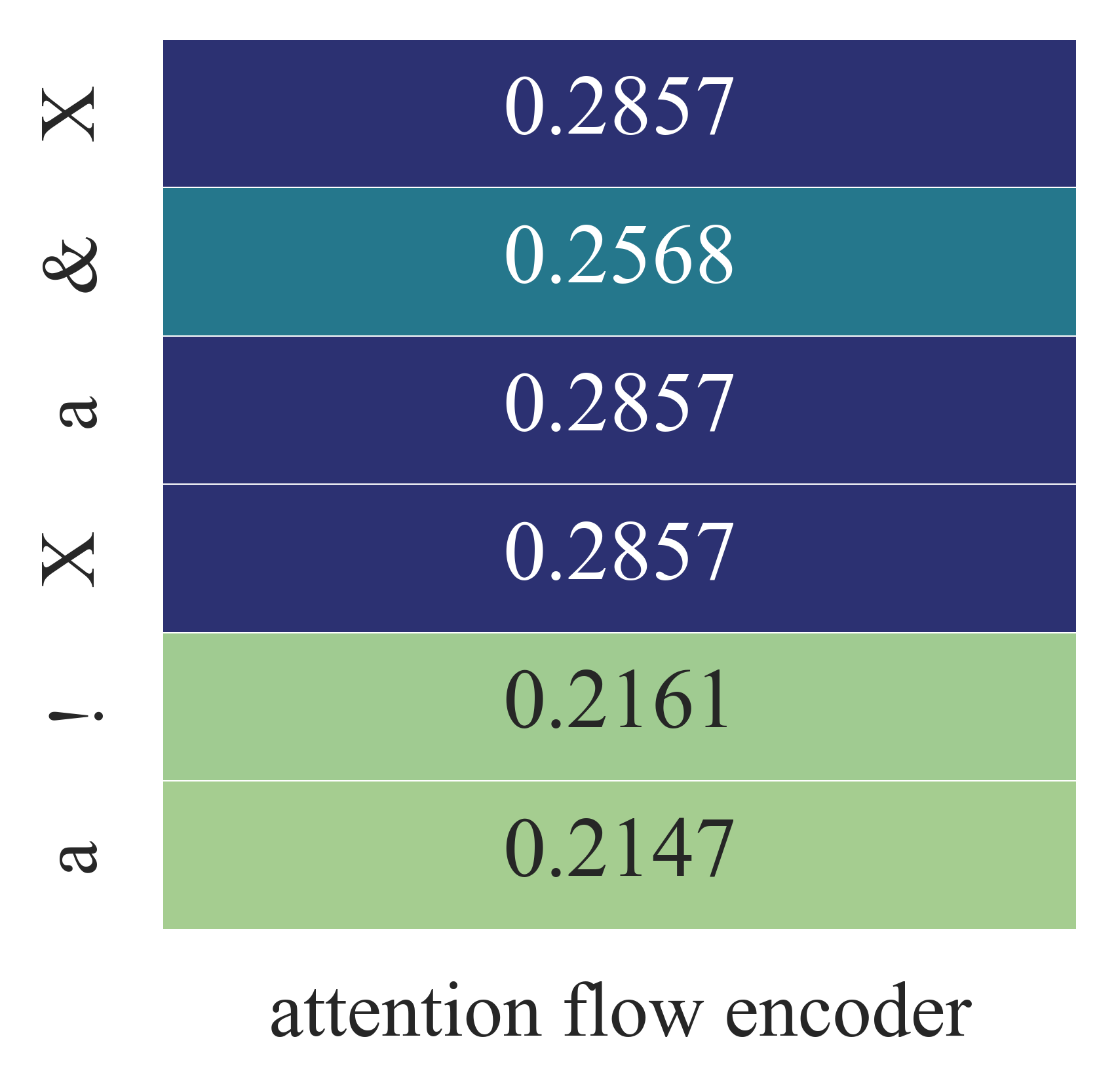}
    \includegraphics[width=.4\linewidth,valign=t]{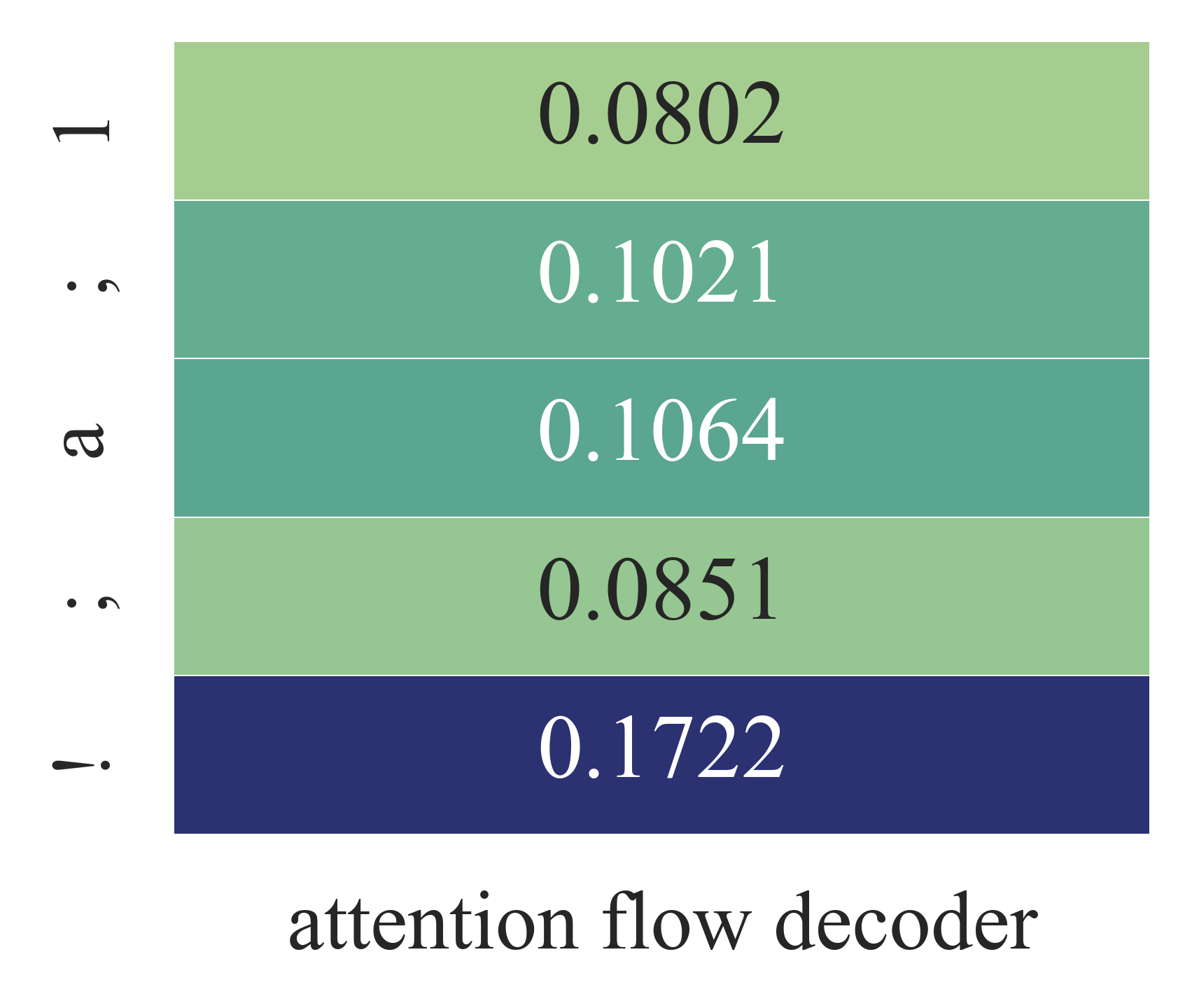}
    \includegraphics[width=.4\linewidth,valign=t]{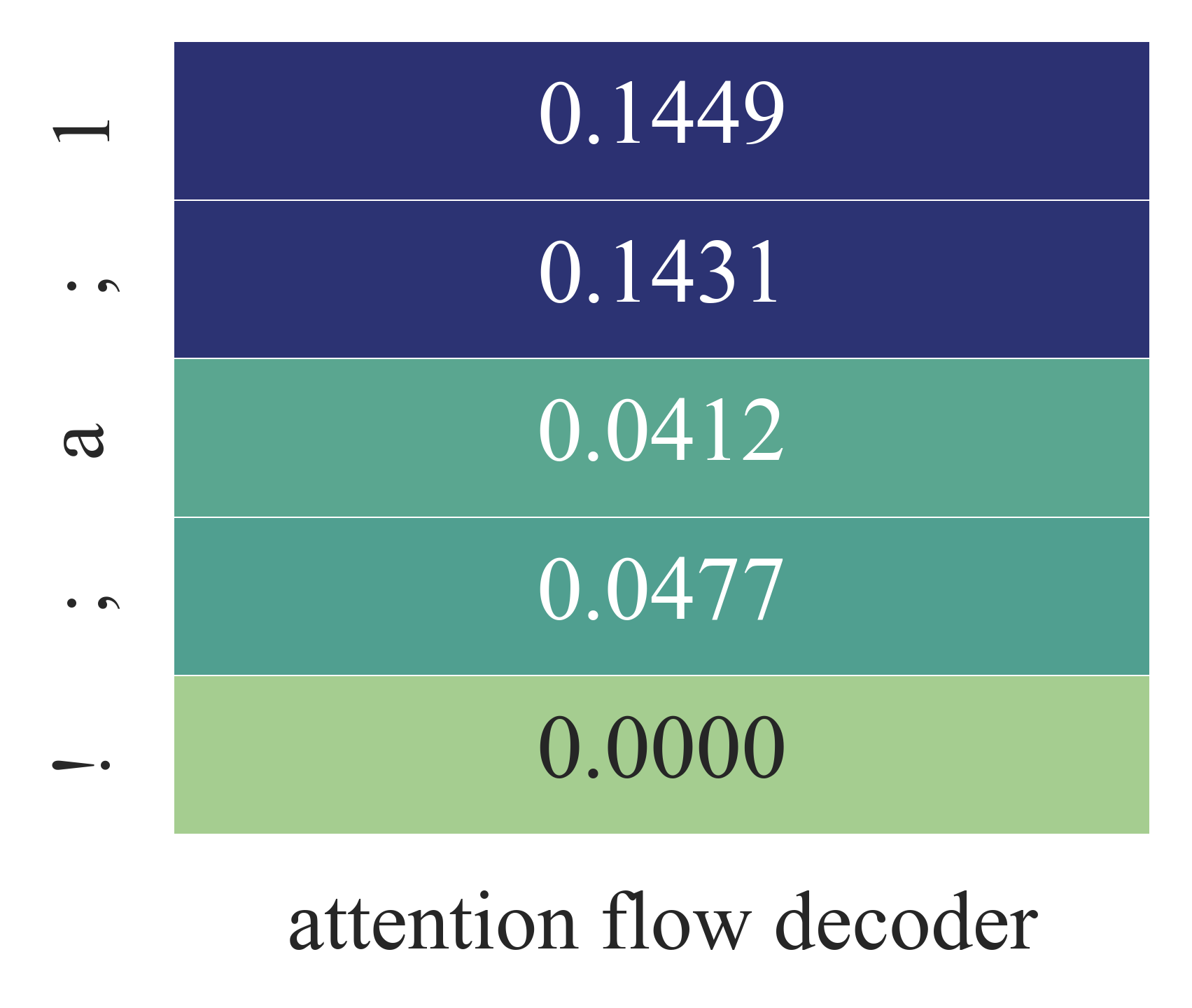}
    \caption{}
    \label{fig:LTLSat}
    \end{subfigure}
    \caption{Heatmap showing the attention flow for 3 variations of the same sentence in RoBERTa in (a) and Heatmaps for two heads of the LTLSat model, each attending a different timestep in (b).}
\end{figure}

\textbf{First token bias.}
While analyzing the attention flow of the decoder-only Transformer DialogPT~\cite{DBLP:conf/acl/ZhangSGCBGGLD20} and GPT-2~\cite{radford2019language}, we observed a heavy bias towards the first decoded token (see Figure~\ref{fig:GPT2:flow:changes} and Figure~\ref{fig:DialogPT:bias} in the appendix).
Regardless of the input tokens, the first token contributes the most to the total attention flow.
Since the DialogPT model was trained on a dataset mined from \url{reddit.com}, it might be beneficial to overattend the first token as many conversations on \url{reddit.com} consist of very short sentences or even single words.
When applying this model outside of similar domains, however, one should be aware of this bias.

\subsection{Head Task Analysis}
\textbf{LTL trace prediction.}
We conducted an experiment in predicting satisfying traces to linear-time temporal logic (LTL)~\citep{DBLP:conf/focs/Pnueli77}.
We used a Transformer trained on this task by~\citet{DBLP:conf/iclr/HahnSKRF21}.
LTL generalizes propositional logic with temporal operators such as $\LTLnext$ (next) or $\LTLuntil$ (until) and is used to specify the behavior of systems that interact with their environments over time.
An LTL formula is satisfied by a trace, which is an infinite sequence of propositions that hold at discrete timesteps.
We finitely represent satisfying traces to LTL formulas as a prefix, followed by a loop, denoted by curly brackets.
The LTL formula $\LTLnext (a \wedge \LTLnext \neg a)$, for example, denotes that on the second position $a$ must be true and on the third position $a$ must be false.
The model correctly predicts the following trace, where the first position and the loop is arbitrary and hence set to true: $\mathit{trace}: 1 ; a ; \neg a ; \{ 1 \}$.
Figure~\ref{fig:LTLSat} depicts the maxflow computation for two heads.
The left head focuses on the $\LTLnext \neg a$ part of the formula and the third position of the trace where $a$ is \emph{not} allowed to appear.
The right head focuses on the left conjunct $a$ and that it must appear on the second position of the trace.
Another example can be found in Appendix~\ref{app:LTLuntil}.

\textbf{Translation.}
\begin{figure}
\begin{subfigure}[b]{.49\textwidth}
    \centering
    \includegraphics[width=.49\linewidth, valign=t]{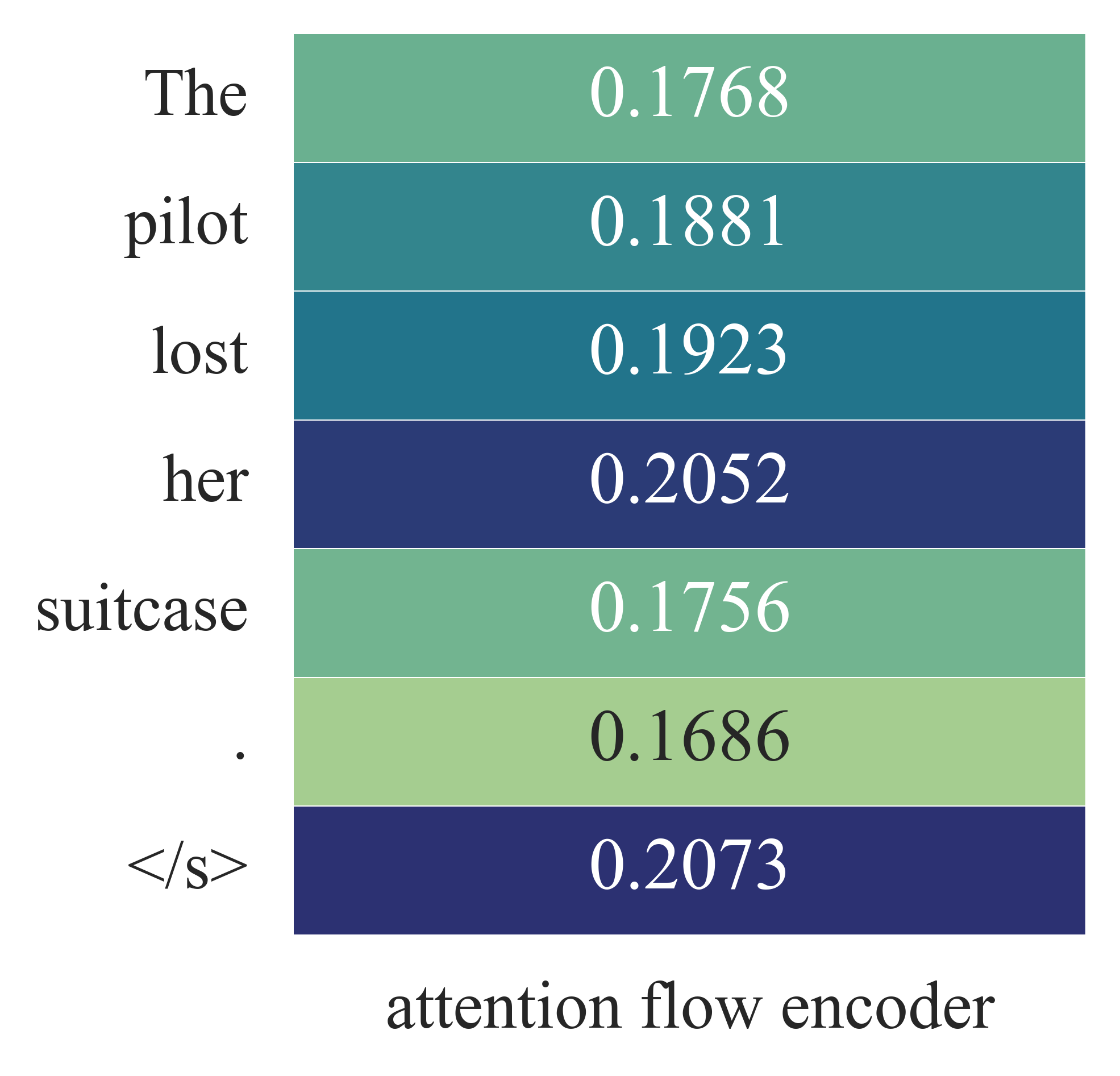}
    \includegraphics[width=.49\linewidth, valign=t]{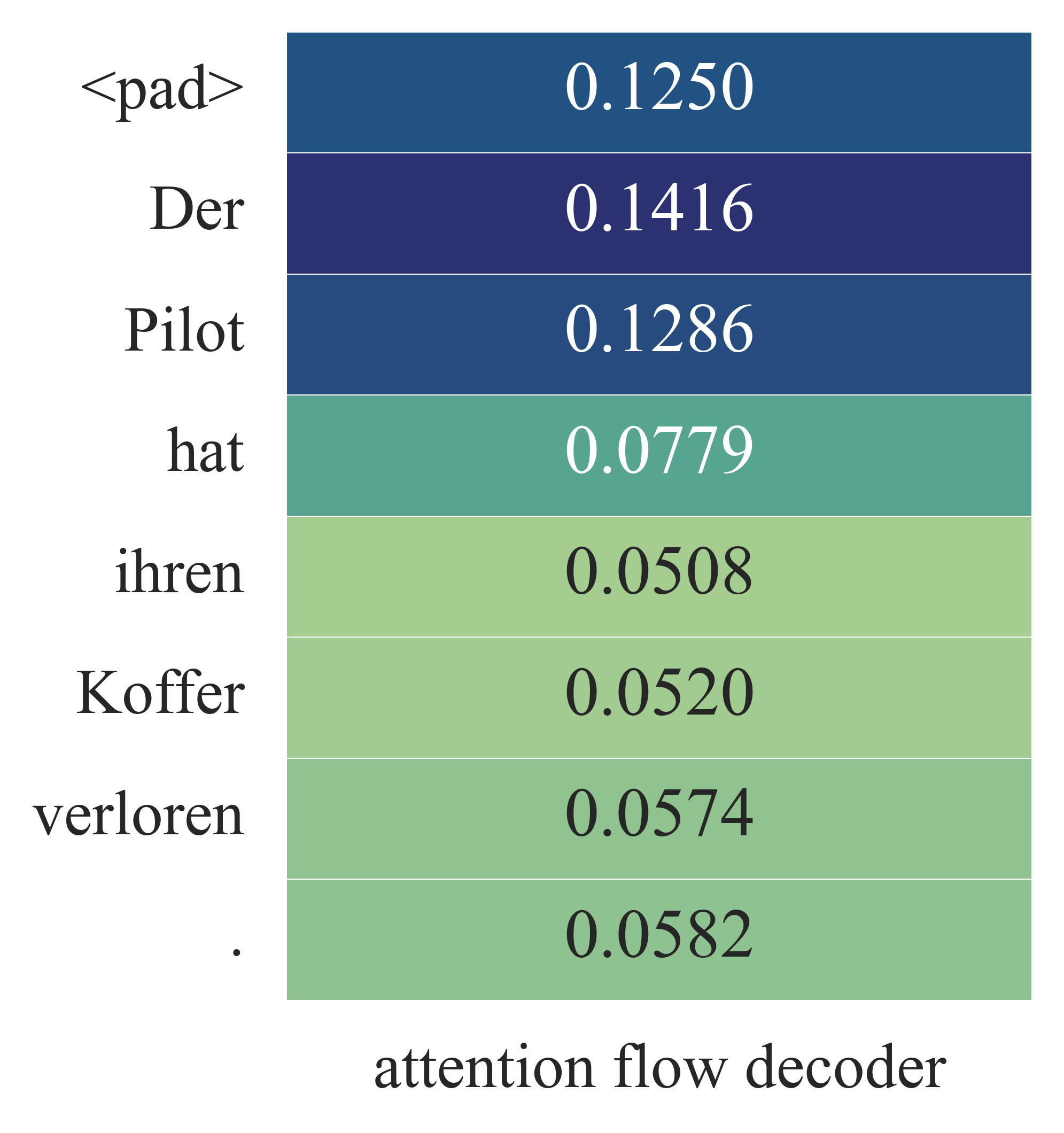}
    \caption{}
\end{subfigure}
\begin{subfigure}[b]{.49\textwidth}
    \includegraphics[width=.49\linewidth, valign=t]{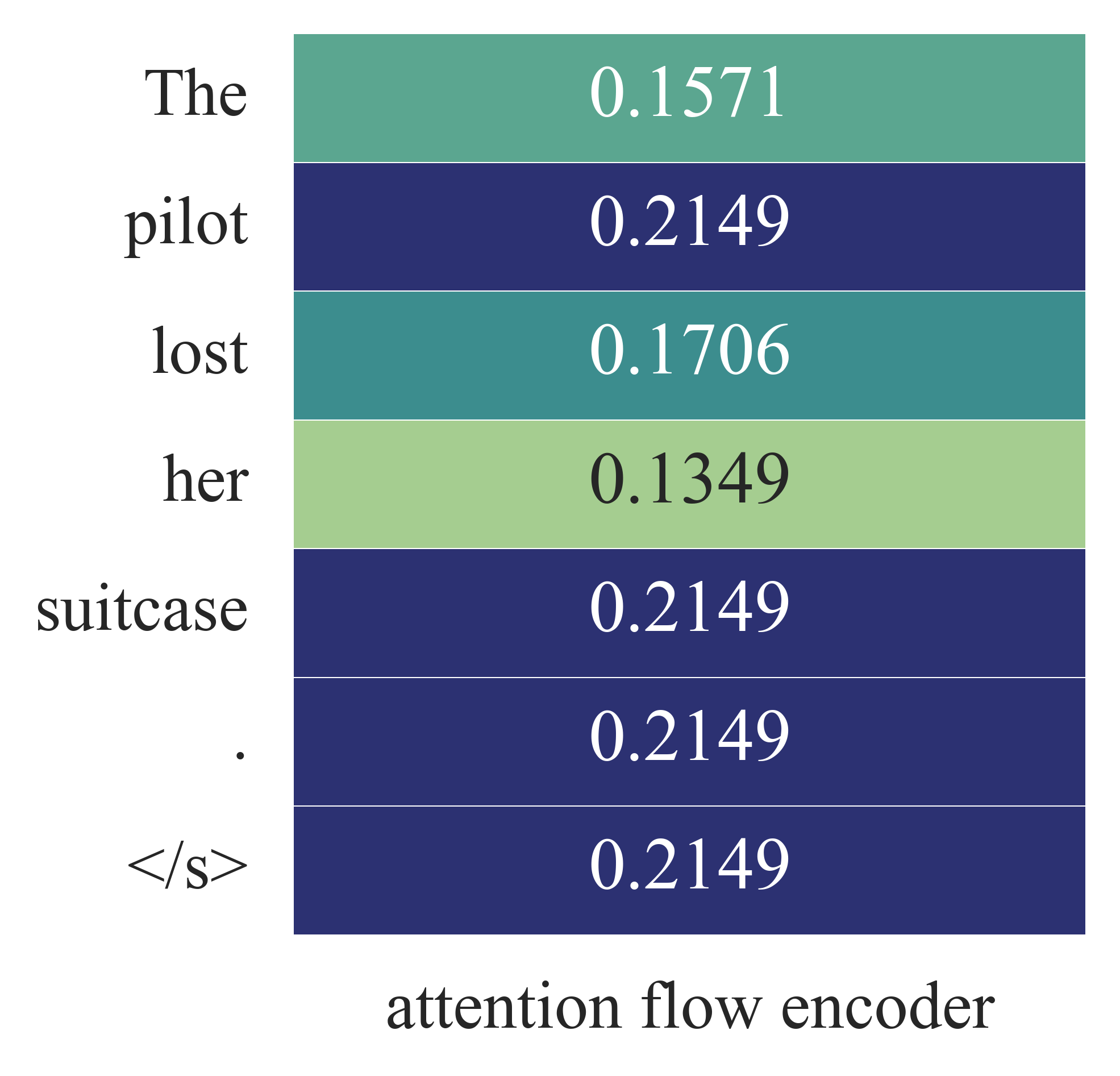}
    \includegraphics[width=.49\linewidth, valign=t]{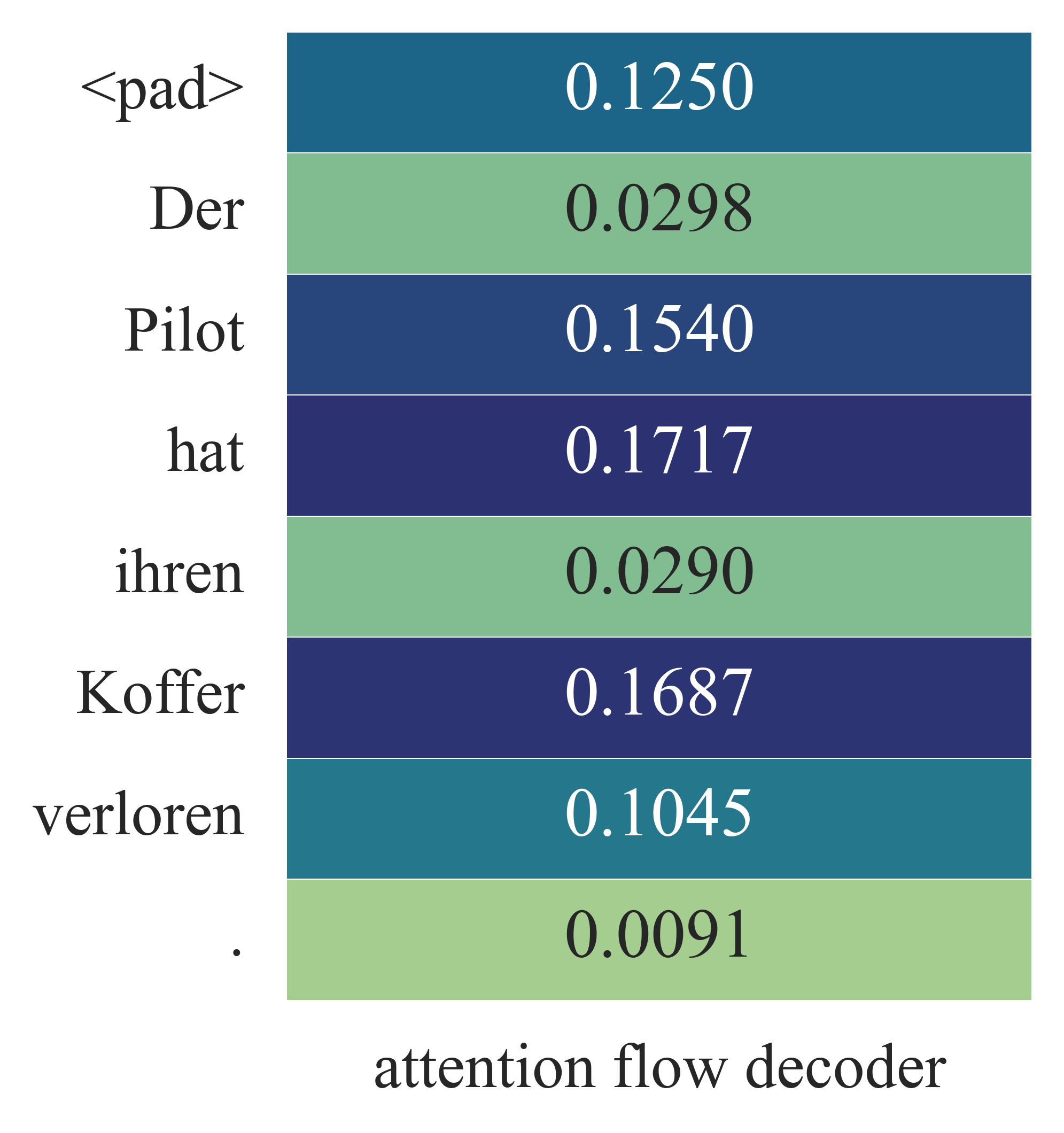}
    \caption{}
\end{subfigure}
    \caption{Heatmap for two heads, divided in encoder and decoder. The left head (a) attends the pilot, the head on the right (b) the suitcase.}
    \label{fig:HelsinkiHeadTask}
\end{figure}
In this experiment, we used the OPUS-MT-EN-DE model~\cite{TiedemannThottingal:EAMT2020} for translating between English and German.
The input sentence is ``The pilot lost her suitcase.'', which is translated to ``Der Pilot hat ihren Koffer verloren''.
The computed flow network can be found in Figure~\ref{fig:crossnetwork}.
While the meaning of the original sentence is ambiguous, as ``the pilot'' could be male or female, the translated sentence is not, since the German phrase \textit{Der Pilot} means a male pilot. It has been conjectured that such gender-biased translations can facilitate problematic stereotypes~\citep{10.5555/3157382.3157584}. Our analysis technique allows further insight into the internal mechanics of the Transformer model in such a scenario.
We analyze the task of the heads, two of them are shown in Figure~\ref{fig:HelsinkiHeadTask}.
By computing the attention flow for encoder and decoder, we can observe that the depicted heads solve opposing tasks: The head on the left hand side attends \textit{pilot lost her} in the encoder and \textit{Der Pilot} in the decoder, which is the one-to-one translation, but without a corresponding possessive pronoun.
The head on the right hand side attends \textit{pilot} and \textit{suitcase} in the encoder and \textit{Pilot hat} as well as \textit{Koffer} in the decoder.
Hence, from the attention flow, we can see that the second head has little influence on the biased translation, as neither \emph{her}, nor \emph{Der} and \emph{ihren} (the German pronoun corresponding to \emph{her}) receive significant attention. Our technique therefore gives us a useful hint that we have to analyze the first head to get to the root of this biased translation.

\section{Limitations and Conclusion}
\label{sec:conclusion}
The main limiting factor of this approach is that the attention flow in a Transformer is the largest, but not the only factor for deciding the next token prediction.
Additionally to the many residual connections (which can be incorporated to the flow networks; see Section~\ref{sec:attn_flow}), Transformer models contain feed-forward networks used as intermediate steps. 
Another, smaller caveat is that flow values cannot be compared across different model architectures as their absolute values have no meaning.
The values can solely be compared to other tokens in the same layer of the same model.
%
Our approach should thus be seen as a valuable addition (and not a replacement) to the large tool box in interpreting machine learning models. 
It  generalizes the efforts in visualizing and interpreting raw attention values.
During our experiments, we found the attention flow values computed with the presented approach very useful in analyzing models, finding biases and fixing respective datasets.

To conclude, we formalized and extended the technique to construct a flow network out of the attention values of encoder-only Transformer models to general Transformer models including an auto-regressive decoder. Running a maxflow algorithm on these constructions returns Shapley values that determine the impact of a token to the total attention flow leading to the decoder's decision. We provide an implementation of our approach that can be applied to arbitrary Transformer models.
Our experiments show the applicability of this analysis method in a variety of application domains. 
We hope that our implementation and the the constructions presented in this paper will aid machine learning practitioners and researchers in the design of reliable and interpretable Transformer models.



\section*{Acknowledgements}
We thank anonymous reviewer \#3 for value feedback on an earlier version of this paper. This work was funded by DFG grant 389792660 as part of \href{https://perspicuous-computing.science}{TRR~248 -- CPEC}, by the European Research Council (ERC)\ Grant OSARES (No. 683300), and by the German Israeli Foundation (GIF) Grant No. I-1513-407./2019.

\bibliographystyle{abbrvnat}
\bibliography{main}

\begin{thebibliography}{47}
\providecommand{\natexlab}[1]{#1}
\providecommand{\url}[1]{\texttt{#1}}
\expandafter\ifx\csname urlstyle\endcsname\relax
  \providecommand{\doi}[1]{doi: #1}\else
  \providecommand{\doi}{doi: \begingroup \urlstyle{rm}\Url}\fi

\bibitem[Aas et~al.(2019)Aas, Jullum, and L{\o}land]{aas2019explaining}
K.~Aas, M.~Jullum, and A.~L{\o}land.
\newblock Explaining individual predictions when features are dependent: More
  accurate approximations to shapley values.
\newblock \emph{arXiv preprint arXiv:1903.10464}, 2019.

\bibitem[Abnar and Zuidema(2020)]{DBLP:conf/acl/AbnarZ20}
S.~Abnar and W.~H. Zuidema.
\newblock Quantifying attention flow in transformers.
\newblock In D.~Jurafsky, J.~Chai, N.~Schluter, and J.~R. Tetreault, editors,
  \emph{Proceedings of the 58th Annual Meeting of the Association for
  Computational Linguistics, {ACL} 2020, Online, July 5-10, 2020}, pages
  4190--4197. Association for Computational Linguistics, 2020.
\newblock \doi{10.18653/v1/2020.acl-main.385}.
\newblock URL \url{https://doi.org/10.18653/v1/2020.acl-main.385}.

\bibitem[Agarwal et~al.(2019)Agarwal, Dhamdhere, and
  Sundararajan]{agarwal2019new}
A.~Agarwal, K.~Dhamdhere, and M.~Sundararajan.
\newblock A new interaction index inspired by the taylor series.
\newblock 2019.

\bibitem[Bahdanau et~al.(2015)Bahdanau, Cho, and
  Bengio]{DBLP:journals/corr/BahdanauCB14}
D.~Bahdanau, K.~Cho, and Y.~Bengio.
\newblock Neural machine translation by jointly learning to align and
  translate.
\newblock In Y.~Bengio and Y.~LeCun, editors, \emph{3rd International
  Conference on Learning Representations, {ICLR} 2015, San Diego, CA, USA, May
  7-9, 2015, Conference Track Proceedings}, 2015.
\newblock URL \url{http://arxiv.org/abs/1409.0473}.

\bibitem[Barbieri et~al.(2020)Barbieri, Camacho{-}Collados, Anke, and
  Neves]{DBLP:conf/emnlp/BarbieriCAN20}
F.~Barbieri, J.~Camacho{-}Collados, L.~E. Anke, and L.~Neves.
\newblock Tweeteval: Unified benchmark and comparative evaluation for tweet
  classification.
\newblock In T.~Cohn, Y.~He, and Y.~Liu, editors, \emph{Findings of the
  Association for Computational Linguistics: {EMNLP} 2020, Online Event, 16-20
  November 2020}, volume {EMNLP} 2020 of \emph{Findings of {ACL}}, pages
  1644--1650. Association for Computational Linguistics, 2020.
\newblock \doi{10.18653/v1/2020.findings-emnlp.148}.
\newblock URL \url{https://doi.org/10.18653/v1/2020.findings-emnlp.148}.

\bibitem[Bolukbasi et~al.(2016)Bolukbasi, Chang, Zou, Saligrama, and
  Kalai]{10.5555/3157382.3157584}
T.~Bolukbasi, K.-W. Chang, J.~Zou, V.~Saligrama, and A.~Kalai.
\newblock Man is to computer programmer as woman is to homemaker? debiasing
  word embeddings.
\newblock In \emph{Proceedings of the 30th International Conference on Neural
  Information Processing Systems}, NIPS'16, page 4356–4364, Red Hook, NY,
  USA, 2016. Curran Associates Inc.
\newblock ISBN 9781510838819.

\bibitem[Brown et~al.(2020)Brown, Mann, Ryder, Subbiah, Kaplan, Dhariwal,
  Neelakantan, Shyam, Sastry, Askell, Agarwal, Herbert{-}Voss, Krueger,
  Henighan, Child, Ramesh, Ziegler, Wu, Winter, Hesse, Chen, Sigler, Litwin,
  Gray, Chess, Clark, Berner, McCandlish, Radford, Sutskever, and
  Amodei]{DBLP:conf/nips/BrownMRSKDNSSAA20}
T.~B. Brown, B.~Mann, N.~Ryder, M.~Subbiah, J.~Kaplan, P.~Dhariwal,
  A.~Neelakantan, P.~Shyam, G.~Sastry, A.~Askell, S.~Agarwal,
  A.~Herbert{-}Voss, G.~Krueger, T.~Henighan, R.~Child, A.~Ramesh, D.~M.
  Ziegler, J.~Wu, C.~Winter, C.~Hesse, M.~Chen, E.~Sigler, M.~Litwin, S.~Gray,
  B.~Chess, J.~Clark, C.~Berner, S.~McCandlish, A.~Radford, I.~Sutskever, and
  D.~Amodei.
\newblock Language models are few-shot learners.
\newblock In H.~Larochelle, M.~Ranzato, R.~Hadsell, M.~Balcan, and H.~Lin,
  editors, \emph{Advances in Neural Information Processing Systems 33: Annual
  Conference on Neural Information Processing Systems 2020, NeurIPS 2020,
  December 6-12, 2020, virtual}, 2020.
\newblock URL
  \url{https://proceedings.neurips.cc/paper/2020/hash/1457c0d6bfcb4967418bfb8ac142f64a-Abstract.html}.

\bibitem[Burkart and Huber(2021)]{burkart2021survey}
N.~Burkart and M.~F. Huber.
\newblock A survey on the explainability of supervised machine learning.
\newblock \emph{Journal of Artificial Intelligence Research}, 70:\penalty0
  245--317, 2021.

\bibitem[Chefer et~al.(2021)Chefer, Gur, and Wolf]{chefer2021transformer}
H.~Chefer, S.~Gur, and L.~Wolf.
\newblock Transformer interpretability beyond attention visualization.
\newblock In \emph{Proceedings of the IEEE/CVF Conference on Computer Vision
  and Pattern Recognition}, pages 782--791, 2021.

\bibitem[Chen et~al.(2021)Chen, Tworek, Jun, Yuan, Pinto, Kaplan, Edwards,
  Burda, Joseph, Brockman, et~al.]{chen2021evaluating}
M.~Chen, J.~Tworek, H.~Jun, Q.~Yuan, H.~P. d.~O. Pinto, J.~Kaplan, H.~Edwards,
  Y.~Burda, N.~Joseph, G.~Brockman, et~al.
\newblock Evaluating large language models trained on code.
\newblock \emph{arXiv preprint arXiv:2107.03374}, 2021.

\bibitem[Collins and Ghahramani(2021)]{LaMBDA}
E.~Collins and Z.~Ghahramani, 2021.
\newblock URL \url{https://blog.google/technology/ai/lamda/}.

\bibitem[Datta et~al.(2016)Datta, Sen, and Zick]{datta2016algorithmic}
A.~Datta, S.~Sen, and Y.~Zick.
\newblock Algorithmic transparency via quantitative input influence: Theory and
  experiments with learning systems.
\newblock In \emph{2016 IEEE symposium on security and privacy (SP)}, pages
  598--617. IEEE, 2016.

\bibitem[Devlin et~al.(2019)Devlin, Chang, Lee, and
  Toutanova]{DBLP:conf/naacl/DevlinCLT19}
J.~Devlin, M.~Chang, K.~Lee, and K.~Toutanova.
\newblock {BERT:} pre-training of deep bidirectional transformers for language
  understanding.
\newblock In J.~Burstein, C.~Doran, and T.~Solorio, editors, \emph{Proceedings
  of the 2019 Conference of the North American Chapter of the Association for
  Computational Linguistics: Human Language Technologies, {NAACL-HLT} 2019,
  Minneapolis, MN, USA, June 2-7, 2019, Volume 1 (Long and Short Papers)},
  pages 4171--4186. Association for Computational Linguistics, 2019.
\newblock \doi{10.18653/v1/n19-1423}.
\newblock URL \url{https://doi.org/10.18653/v1/n19-1423}.

\bibitem[Ethayarajh and Jurafsky(2021)]{DBLP:conf/acl/EthayarajhJ20}
K.~Ethayarajh and D.~Jurafsky.
\newblock Attention flows are shapley value explanations.
\newblock In C.~Zong, F.~Xia, W.~Li, and R.~Navigli, editors, \emph{Proceedings
  of the 59th Annual Meeting of the Association for Computational Linguistics
  and the 11th International Joint Conference on Natural Language Processing,
  {ACL/IJCNLP} 2021, (Volume 2: Short Papers), Virtual Event, August 1-6,
  2021}, pages 49--54. Association for Computational Linguistics, 2021.
\newblock \doi{10.18653/v1/2021.acl-short.8}.
\newblock URL \url{https://doi.org/10.18653/v1/2021.acl-short.8}.

\bibitem[Ghorbani and Zou(2019)]{ghorbani2019data}
A.~Ghorbani and J.~Zou.
\newblock Data shapley: Equitable valuation of data for machine learning.
\newblock In \emph{International Conference on Machine Learning}, pages
  2242--2251. PMLR, 2019.

\bibitem[Gr{\"o}mping(2007)]{gromping2007estimators}
U.~Gr{\"o}mping.
\newblock Estimators of relative importance in linear regression based on
  variance decomposition.
\newblock \emph{The American Statistician}, 61\penalty0 (2):\penalty0 139--147,
  2007.

\bibitem[Hagberg et~al.(2008)Hagberg, Swart, and S~Chult]{networkx}
A.~Hagberg, P.~Swart, and D.~S~Chult.
\newblock Exploring network structure, dynamics, and function using networkx.
\newblock Technical report, Los Alamos National Lab.(LANL), Los Alamos, NM
  (United States), 2008.

\bibitem[Hahn et~al.(2021)Hahn, Schmitt, Kreber, Rabe, and
  Finkbeiner]{DBLP:conf/iclr/HahnSKRF21}
C.~Hahn, F.~Schmitt, J.~U. Kreber, M.~N. Rabe, and B.~Finkbeiner.
\newblock Teaching temporal logics to neural networks.
\newblock In \emph{9th International Conference on Learning Representations,
  {ICLR} 2021, Virtual Event, Austria, May 3-7, 2021}. OpenReview.net, 2021.
\newblock URL \url{https://openreview.net/forum?id=dOcQK-f4byz}.

\bibitem[Han et~al.(2021)Han, Rute, Wu, Ayers, and Polu]{han2021proof}
J.~M. Han, J.~Rute, Y.~Wu, E.~W. Ayers, and S.~Polu.
\newblock Proof artifact co-training for theorem proving with language models.
\newblock \emph{arXiv preprint arXiv:2102.06203}, 2021.

\bibitem[Harris and Ross(1955)]{harris1955fundamentals}
T.~Harris and F.~Ross.
\newblock Fundamentals of a method for evaluating rail net capacities.
\newblock Technical report, RAND CORP SANTA MONICA CA, 1955.

\bibitem[Khan et~al.(2021)Khan, Naseer, Hayat, Zamir, Khan, and
  Shah]{khan2021transformers}
S.~Khan, M.~Naseer, M.~Hayat, S.~W. Zamir, F.~S. Khan, and M.~Shah.
\newblock Transformers in vision: A survey.
\newblock \emph{arXiv preprint arXiv:2101.01169}, 2021.

\bibitem[Lample and Charton(2019)]{lample2019deep}
G.~Lample and F.~Charton.
\newblock Deep learning for symbolic mathematics.
\newblock \emph{arXiv preprint arXiv:1912.01412}, 2019.

\bibitem[Lindeman(1980)]{lindeman1980introduction}
R.~H. Lindeman.
\newblock Introduction to bivariate and multivariate analysis.
\newblock Technical report, 1980.

\bibitem[Liu et~al.(2019)Liu, Ott, Goyal, Du, Joshi, Chen, Levy, Lewis,
  Zettlemoyer, and Stoyanov]{DBLP:journals/corr/abs-1907-11692}
Y.~Liu, M.~Ott, N.~Goyal, J.~Du, M.~Joshi, D.~Chen, O.~Levy, M.~Lewis,
  L.~Zettlemoyer, and V.~Stoyanov.
\newblock Roberta: {A} robustly optimized {BERT} pretraining approach.
\newblock \emph{CoRR}, abs/1907.11692, 2019.
\newblock URL \url{http://arxiv.org/abs/1907.11692}.

\bibitem[Lundberg and Lee(2017)]{NIPS2017_7062}
S.~M. Lundberg and S.-I. Lee.
\newblock A unified approach to interpreting model predictions.
\newblock In I.~Guyon, U.~V. Luxburg, S.~Bengio, H.~Wallach, R.~Fergus,
  S.~Vishwanathan, and R.~Garnett, editors, \emph{Advances in Neural
  Information Processing Systems 30}, pages 4765--4774. Curran Associates,
  Inc., 2017.
\newblock URL
  \url{http://papers.nips.cc/paper/7062-a-unified-approach-to-interpreting-model-predictions.pdf}.

\bibitem[Lundberg et~al.(2018)Lundberg, Erion, and Lee]{lundberg2018consistent}
S.~M. Lundberg, G.~G. Erion, and S.-I. Lee.
\newblock Consistent individualized feature attribution for tree ensembles.
\newblock \emph{arXiv preprint arXiv:1802.03888}, 2018.

\bibitem[Montavon et~al.(2017)Montavon, Lapuschkin, Binder, Samek, and
  M{\"u}ller]{montavon2017explaining}
G.~Montavon, S.~Lapuschkin, A.~Binder, W.~Samek, and K.-R. M{\"u}ller.
\newblock Explaining nonlinear classification decisions with deep taylor
  decomposition.
\newblock \emph{Pattern Recognition}, 65:\penalty0 211--222, 2017.

\bibitem[Owen(2014)]{owen2014sobol}
A.~B. Owen.
\newblock Sobol'indices and shapley value.
\newblock \emph{SIAM/ASA Journal on Uncertainty Quantification}, 2\penalty0
  (1):\penalty0 245--251, 2014.

\bibitem[Owen and Prieur(2017)]{owen2017shapley}
A.~B. Owen and C.~Prieur.
\newblock On shapley value for measuring importance of dependent inputs.
\newblock \emph{SIAM/ASA Journal on Uncertainty Quantification}, 5\penalty0
  (1):\penalty0 986--1002, 2017.

\bibitem[Pnueli(1977)]{DBLP:conf/focs/Pnueli77}
A.~Pnueli.
\newblock The temporal logic of programs.
\newblock In \emph{18th Annual Symposium on Foundations of Computer Science,
  Providence, Rhode Island, USA, 31 October - 1 November 1977}, pages 46--57.
  {IEEE} Computer Society, 1977.
\newblock \doi{10.1109/SFCS.1977.32}.
\newblock URL \url{https://doi.org/10.1109/SFCS.1977.32}.

\bibitem[Radford et~al.(2018)Radford, Narasimhan, Salimans, and
  Sutskever]{radford2018improving}
A.~Radford, K.~Narasimhan, T.~Salimans, and I.~Sutskever.
\newblock Improving language understanding by generative pre-training.
\newblock 2018.

\bibitem[Radford et~al.(2019)Radford, Wu, Child, Luan, Amodei, Sutskever,
  et~al.]{radford2019language}
A.~Radford, J.~Wu, R.~Child, D.~Luan, D.~Amodei, I.~Sutskever, et~al.
\newblock Language models are unsupervised multitask learners.
\newblock \emph{OpenAI blog}, 1\penalty0 (8):\penalty0 9, 2019.

\bibitem[Samek et~al.(2017)Samek, Wiegand, and
  M{\"u}ller]{samek2017explainable}
W.~Samek, T.~Wiegand, and K.-R. M{\"u}ller.
\newblock Explainable artificial intelligence: Understanding, visualizing and
  interpreting deep learning models.
\newblock \emph{arXiv preprint arXiv:1708.08296}, 2017.

\bibitem[Schmitt et~al.(2021)Schmitt, Hahn, Rabe, and
  Finkbeiner]{DBLP:journals/corr/abs-2107-11864}
F.~Schmitt, C.~Hahn, M.~N. Rabe, and B.~Finkbeiner.
\newblock Neural circuit synthesis from specification patterns.
\newblock \emph{CoRR}, abs/2107.11864, 2021.
\newblock URL \url{https://arxiv.org/abs/2107.11864}.

\bibitem[Shapley(1953)]{shapley201617}
L.~Shapley.
\newblock A value fo n-person games.
\newblock \emph{Ann. Math. Study28, Contributions to the Theory of Games, ed.
  by HW Kuhn, and AW Tucker}, pages 307--317, 1953.

\bibitem[{\v{S}}trumbelj and Kononenko(2014)]{vstrumbelj2014explaining}
E.~{\v{S}}trumbelj and I.~Kononenko.
\newblock Explaining prediction models and individual predictions with feature
  contributions.
\newblock \emph{Knowledge and information systems}, 41\penalty0 (3):\penalty0
  647--665, 2014.

\bibitem[{\v{S}}trumbelj et~al.(2009){\v{S}}trumbelj, Kononenko, and
  {\v{S}}ikonja]{vstrumbelj2009explaining}
E.~{\v{S}}trumbelj, I.~Kononenko, and M.~R. {\v{S}}ikonja.
\newblock Explaining instance classifications with interactions of subsets of
  feature values.
\newblock \emph{Data \& Knowledge Engineering}, 68\penalty0 (10):\penalty0
  886--904, 2009.

\bibitem[Sun and Sundararajan(2011)]{sun2011axiomatic}
Y.~Sun and M.~Sundararajan.
\newblock Axiomatic attribution for multilinear functions.
\newblock In \emph{Proceedings of the 12th ACM conference on Electronic
  commerce}, pages 177--178, 2011.

\bibitem[Sundararajan and Najmi(2020)]{sundararajan2020many}
M.~Sundararajan and A.~Najmi.
\newblock The many shapley values for model explanation.
\newblock In \emph{International Conference on Machine Learning}, pages
  9269--9278. PMLR, 2020.

\bibitem[Sundararajan et~al.(2017)Sundararajan, Taly, and
  Yan]{sundararajan2017axiomatic}
M.~Sundararajan, A.~Taly, and Q.~Yan.
\newblock Axiomatic attribution for deep networks.
\newblock In \emph{International Conference on Machine Learning}, pages
  3319--3328. PMLR, 2017.

\bibitem[Tiedemann and Thottingal(2020)]{TiedemannThottingal:EAMT2020}
J.~Tiedemann and S.~Thottingal.
\newblock {OPUS-MT} — {B}uilding open translation services for the {W}orld.
\newblock In \emph{Proceedings of the 22nd Annual Conferenec of the European
  Association for Machine Translation (EAMT)}, Lisbon, Portugal, 2020.

\bibitem[Vaswani et~al.(2017)Vaswani, Shazeer, Parmar, Uszkoreit, Jones, Gomez,
  Kaiser, and Polosukhin]{DBLP:conf/nips/VaswaniSPUJGKP17}
A.~Vaswani, N.~Shazeer, N.~Parmar, J.~Uszkoreit, L.~Jones, A.~N. Gomez,
  L.~Kaiser, and I.~Polosukhin.
\newblock Attention is all you need.
\newblock In I.~Guyon, U.~von Luxburg, S.~Bengio, H.~M. Wallach, R.~Fergus,
  S.~V.~N. Vishwanathan, and R.~Garnett, editors, \emph{Advances in Neural
  Information Processing Systems 30: Annual Conference on Neural Information
  Processing Systems 2017, December 4-9, 2017, Long Beach, CA, {USA}}, pages
  5998--6008, 2017.
\newblock URL
  \url{https://proceedings.neurips.cc/paper/2017/hash/3f5ee243547dee91fbd053c1c4a845aa-Abstract.html}.

\bibitem[Vig(2019)]{vig2019multiscale}
J.~Vig.
\newblock A multiscale visualization of attention in the transformer model.
\newblock \emph{arXiv preprint arXiv:1906.05714}, 2019.

\bibitem[Wang et~al.(2021)Wang, Turko, and Chau]{wang2021dodrio}
Z.~J. Wang, R.~Turko, and D.~H. Chau.
\newblock Dodrio: Exploring transformer models with interactive visualization.
\newblock \emph{arXiv preprint arXiv:2103.14625}, 2021.

\bibitem[Waskom(2021)]{Waskom2021}
M.~L. Waskom.
\newblock seaborn: statistical data visualization.
\newblock \emph{Journal of Open Source Software}, 6\penalty0 (60):\penalty0
  3021, 2021.
\newblock \doi{10.21105/joss.03021}.
\newblock URL \url{https://doi.org/10.21105/joss.03021}.

\bibitem[Wolf et~al.(2019)Wolf, Debut, Sanh, Chaumond, Delangue, Moi, Cistac,
  Rault, Louf, Funtowicz, et~al.]{wolf2019huggingface}
T.~Wolf, L.~Debut, V.~Sanh, J.~Chaumond, C.~Delangue, A.~Moi, P.~Cistac,
  T.~Rault, R.~Louf, M.~Funtowicz, et~al.
\newblock Huggingface's transformers: State-of-the-art natural language
  processing.
\newblock \emph{arXiv preprint arXiv:1910.03771}, 2019.

\bibitem[Zhang et~al.(2020)Zhang, Sun, Galley, Chen, Brockett, Gao, Gao, Liu,
  and Dolan]{DBLP:conf/acl/ZhangSGCBGGLD20}
Y.~Zhang, S.~Sun, M.~Galley, Y.~Chen, C.~Brockett, X.~Gao, J.~Gao, J.~Liu, and
  B.~Dolan.
\newblock {DIALOGPT} : Large-scale generative pre-training for conversational
  response generation.
\newblock In A.~Celikyilmaz and T.~Wen, editors, \emph{Proceedings of the 58th
  Annual Meeting of the Association for Computational Linguistics: System
  Demonstrations, {ACL} 2020, Online, July 5-10, 2020}, pages 270--278.
  Association for Computational Linguistics, 2020.
\newblock \doi{10.18653/v1/2020.acl-demos.30}.
\newblock URL \url{https://doi.org/10.18653/v1/2020.acl-demos.30}.

\end{thebibliography}

\clearpage
\appendix
\section{Head Task Analysis: LTL until-operator}
\label{app:LTLuntil}
In this experiment, we provide another LTL example where one of the heads is focusing on the temporal operator in the formula and another is focusing solely on the propositions of the formula~(see Figure~\ref{fig:LTLSatUntil})
The input formula is: $a \LTLuntil b \wedge 1 \LTLuntil a$, where $1 \LTLuntil a$ denotes that finally an $a$ must occur. The network correctly outputs the following trace: $\mathit{trace:}~a \wedge b; \{ 1 \}$.

\begin{figure}[h]
    \centering
    \includegraphics[width=.4\linewidth, valign=b]{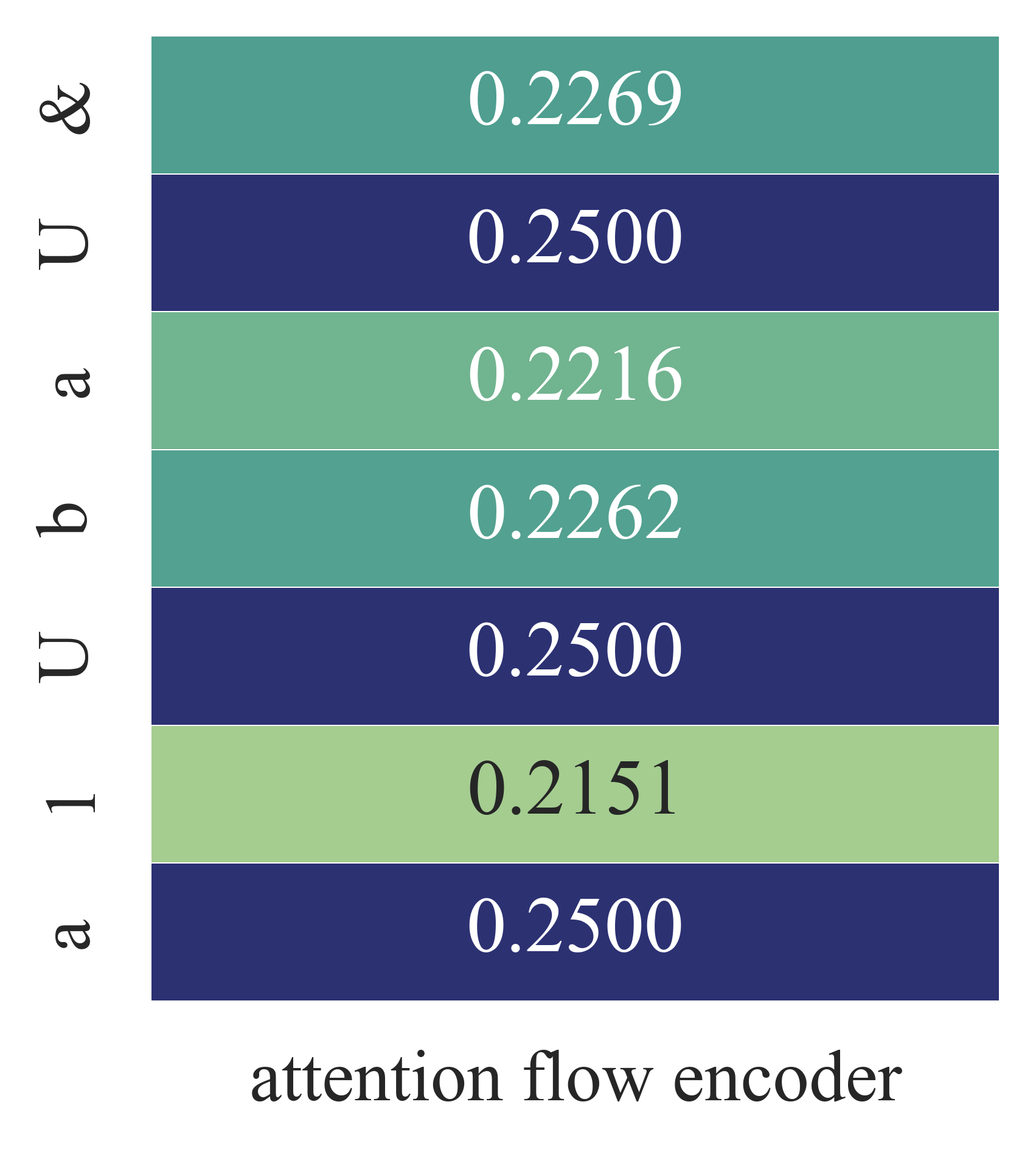}
    \includegraphics[width=.4\linewidth, valign=b]{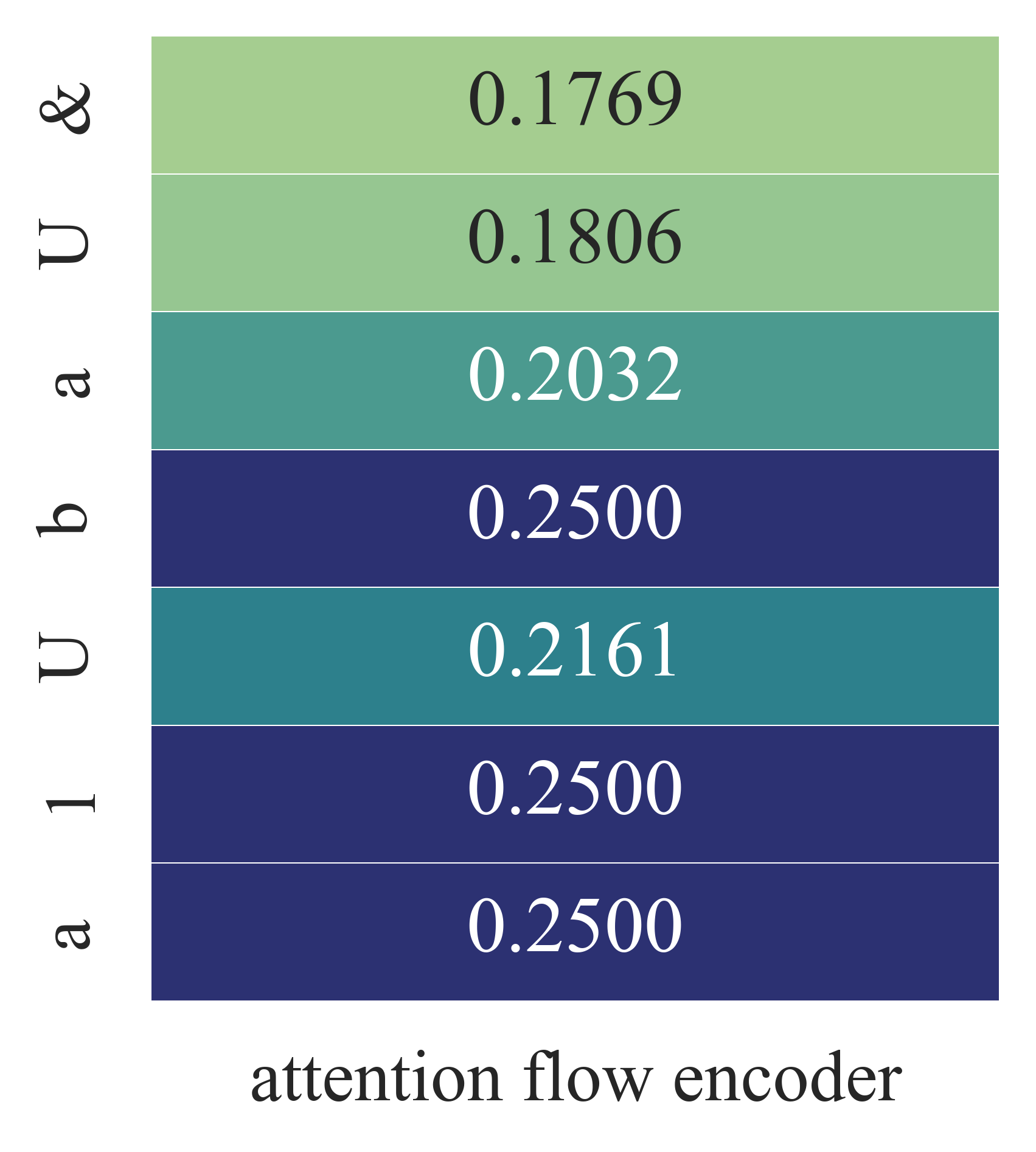}
    \caption{Heatmap of the attention flow for two heads: The left head focuses on the until-operator and the right head focuses on the propositions.}
    \label{fig:LTLSatUntil}
\end{figure}
\section{Additional Figures}

\subsection{Heatmaps} 
\paragraph{Single Head Attention Flow in RoBERTa.}
Figure~\ref{fig:Head0:RoBERTa} depicts the attention flow of the first head in the RoBERTa model. 
Intuitively, the word \textit{killer} dominates the sentiment of the sentence.
However, the output of RoBERTa is a positive sentiment, although the attention flow is mainly on the word \textit{killer} (see Figure~\ref{fig:John}).
Analyzing the individual heads, one can observe that head 0 attends the smiley with its maximal value (1.0), which could be one explanation for the output of the model.

\begin{figure}[h]
    \centering
    \includegraphics[width=.5\linewidth, valign=b]{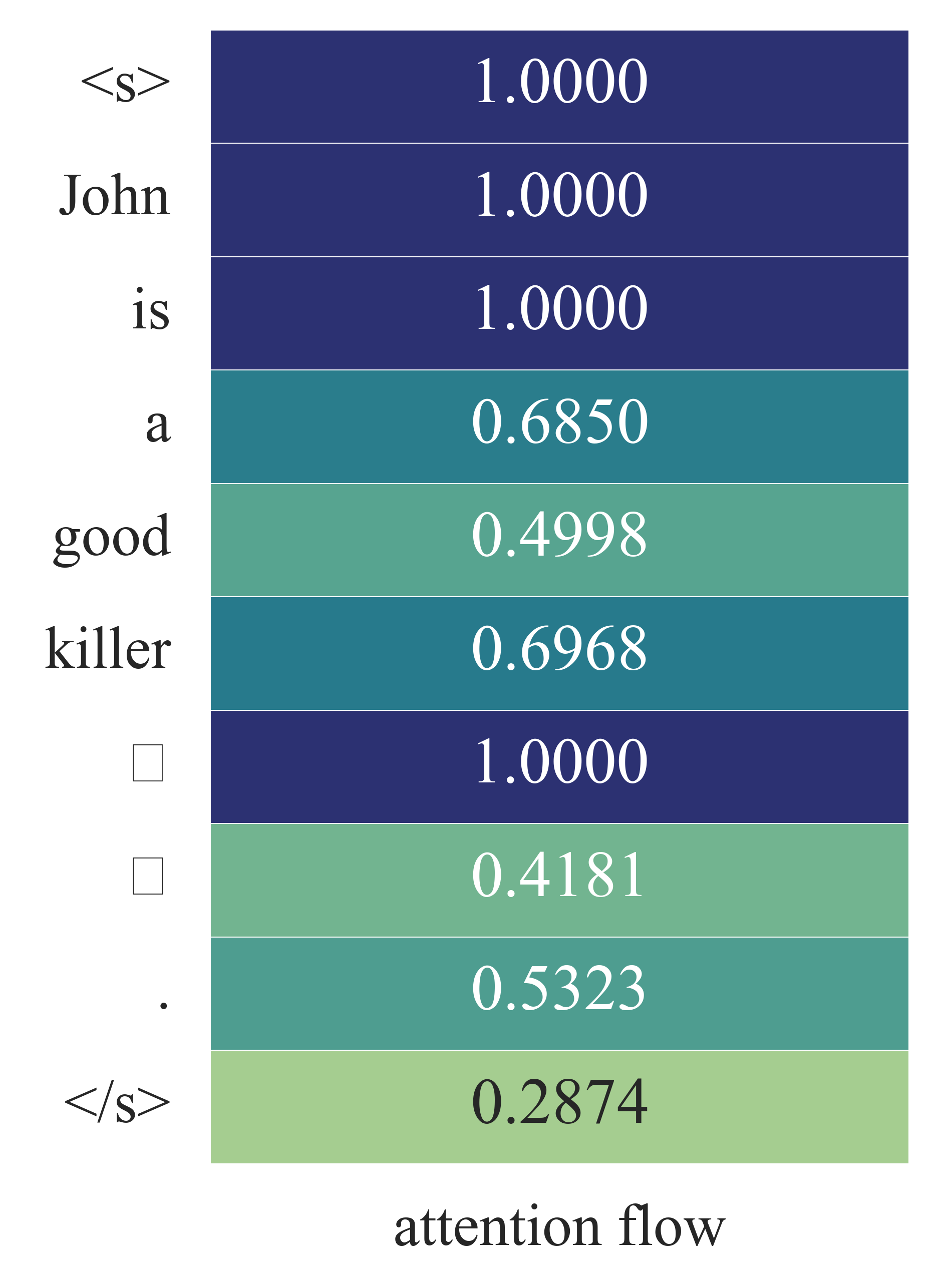}
    \caption{Heatmap showing head 0 of RoBERTa for the example in Figure~\ref{fig:John}.}
    \label{fig:Head0:RoBERTa}
\end{figure}

\paragraph{Bias in DialogPT.}
Figure~\ref{fig:DialogPT:bias} shows the attention flow from each token to the current output.
While we observe slight changes of the computed attention flow for each token, the first input token \textit{The} is highly attended, more than two times the attention flow than any other token.
Note that this observation does not directly translate into a bias in the model, it solely shows that the distribution of attention is biased.

\begin{figure}[h]
    \centering
    \includegraphics[width=.5\linewidth, valign=b]{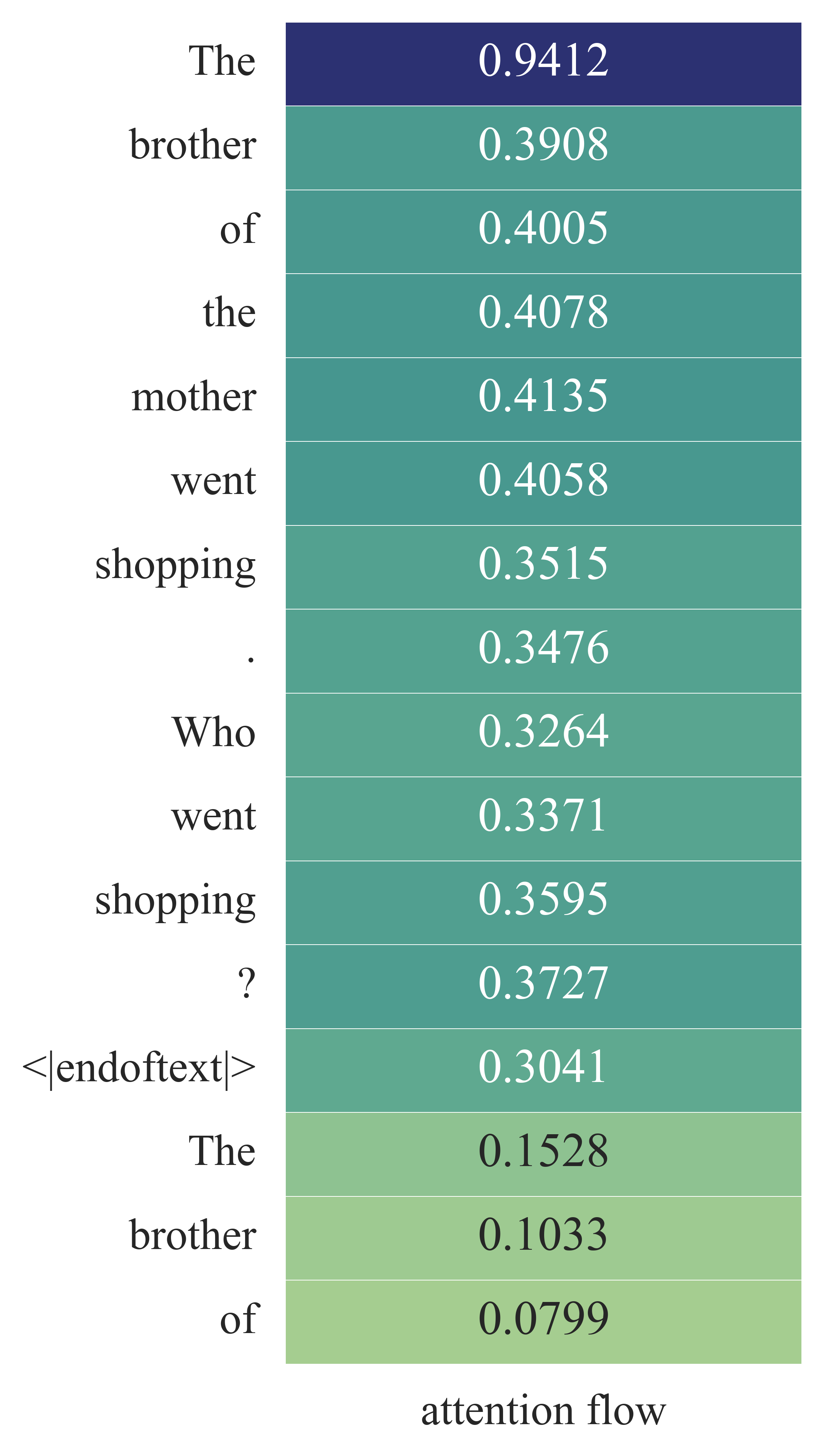}
    \caption{Heatmap showing the bias towards the first token in DialogPT.}
    \label{fig:DialogPT:bias}
\end{figure}

\subsection{Flow Networks}
\paragraph{Encoder Only.}
Figure~\ref{fig:encodernetwork} shows the flow network for RoBERTa.
The underlying architecture is encoder only with 12 layers, represented by the 12 layers of attention edges between the nodes, and an input sentence with 10 tokens, represented on the y-axis of the network.
The special property of RoBERTa is the classification token at position 0 - only the attention flow to this token in the last node layer is important.

\paragraph{Encoder Decoder.}
Figure~\ref{fig:crossnetwork} shows the flow network for OPUS-MT-EN-DE.
The underlying architecture consists of an encoder and a decoder with 8 layers each, connected by the cross attention edges in between.
For each input token and auto regressive token, we compute the attention flow to each predicted token.
In Figure~\ref{fig:crossnetwork}, the attention flow for the third predicted token is computed.

\paragraph{Decoder Only.}
Figure~\ref{fig:decodernetwork} shows the flow network for GPT-2 with the underlying decoder only architecture.
The model has 12 layers, attention can only flow from previous auto regressive tokens, including the input tokens.
We start computing attention flow for the first output token, which is connected to the terminal node in Figure~\ref{fig:decodernetwork}.
\begin{figure*}
    \centering
    \includegraphics[width=\textwidth]{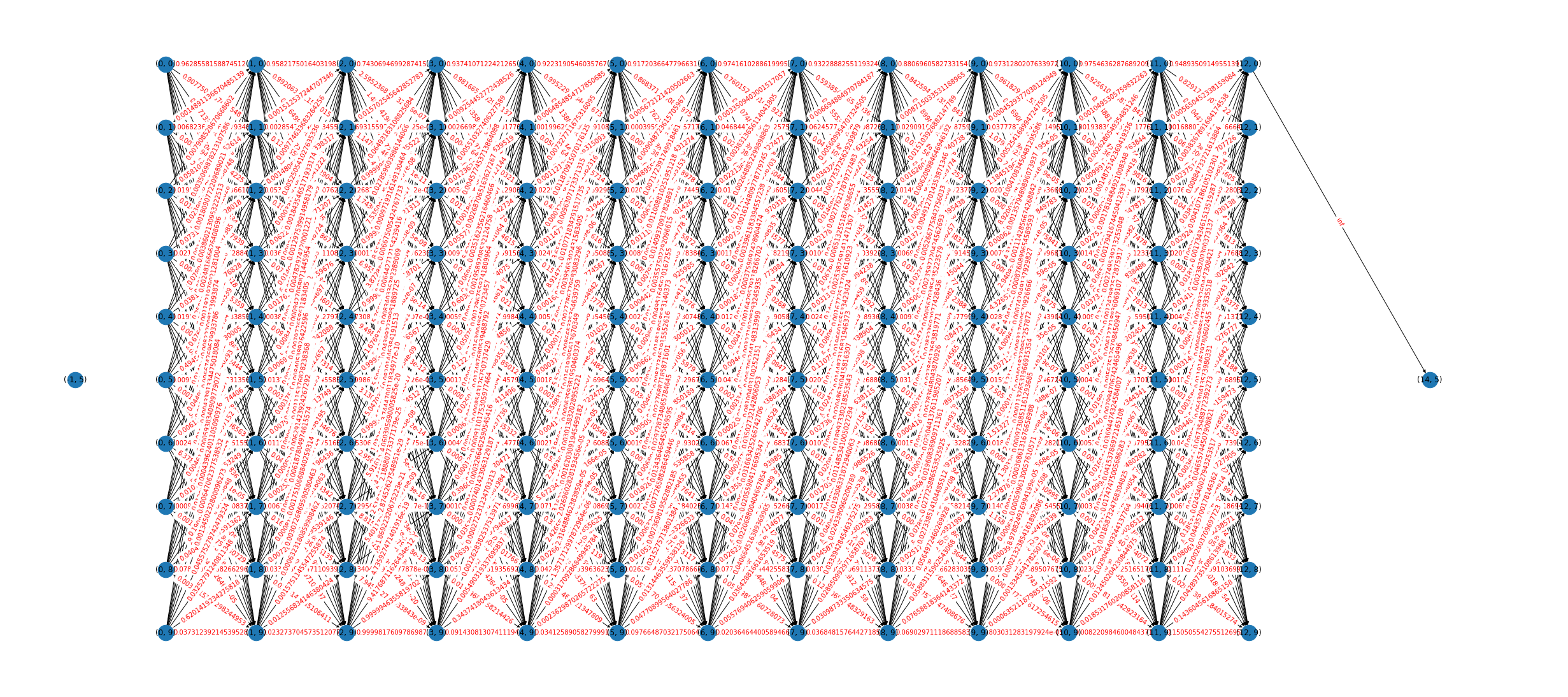}
    \caption{The flow network of the encoder-only network RoBERTa for the example in Figure~\ref{fig:John}.}
    \label{fig:encodernetwork}
\end{figure*}
\begin{figure*}
    \centering
    \includegraphics[width=\textwidth]{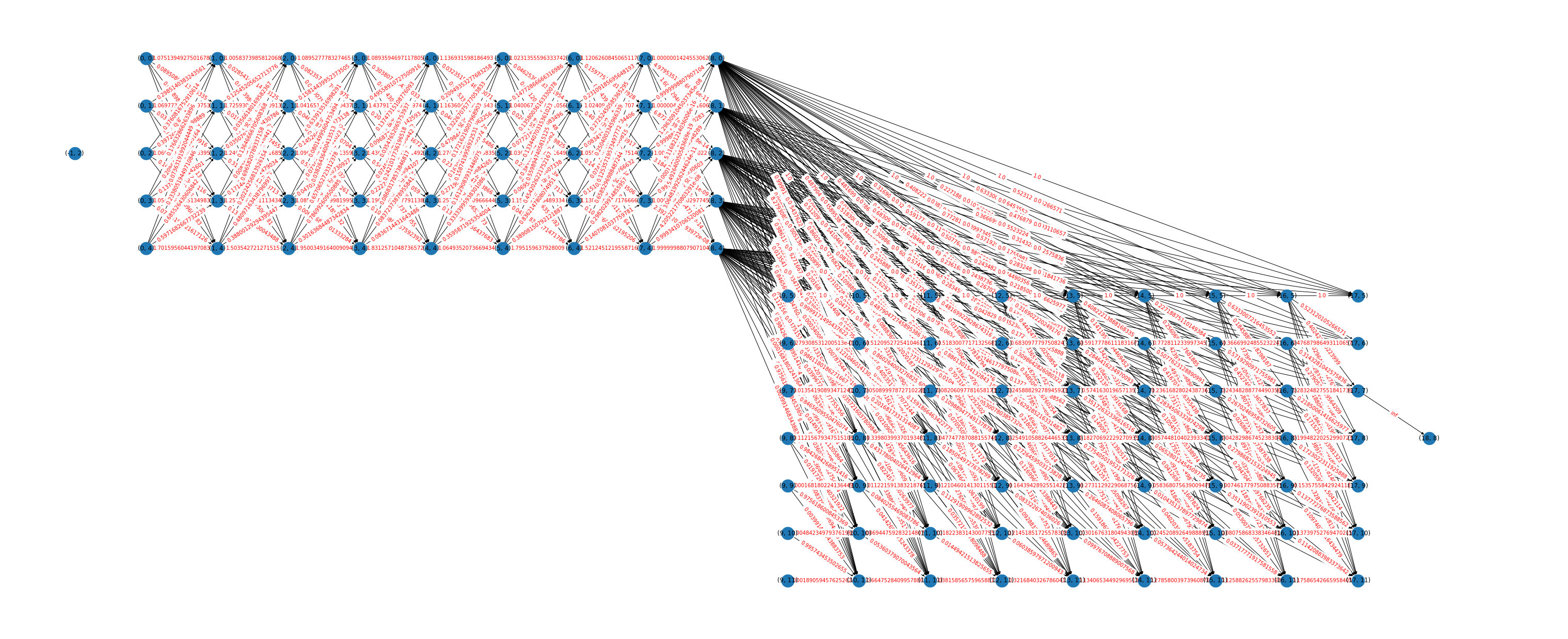}
    \caption{The flow network of the  encoder decoder architecture OPUS-MT-EN-DE for the input ``The father cooked dinner.'' and the predicted tokens ``Der Vater kochte Abendessen''.}
    \label{fig:crossnetwork}
\end{figure*}
\begin{figure*}
    \centering
    \includegraphics[width=\textwidth]{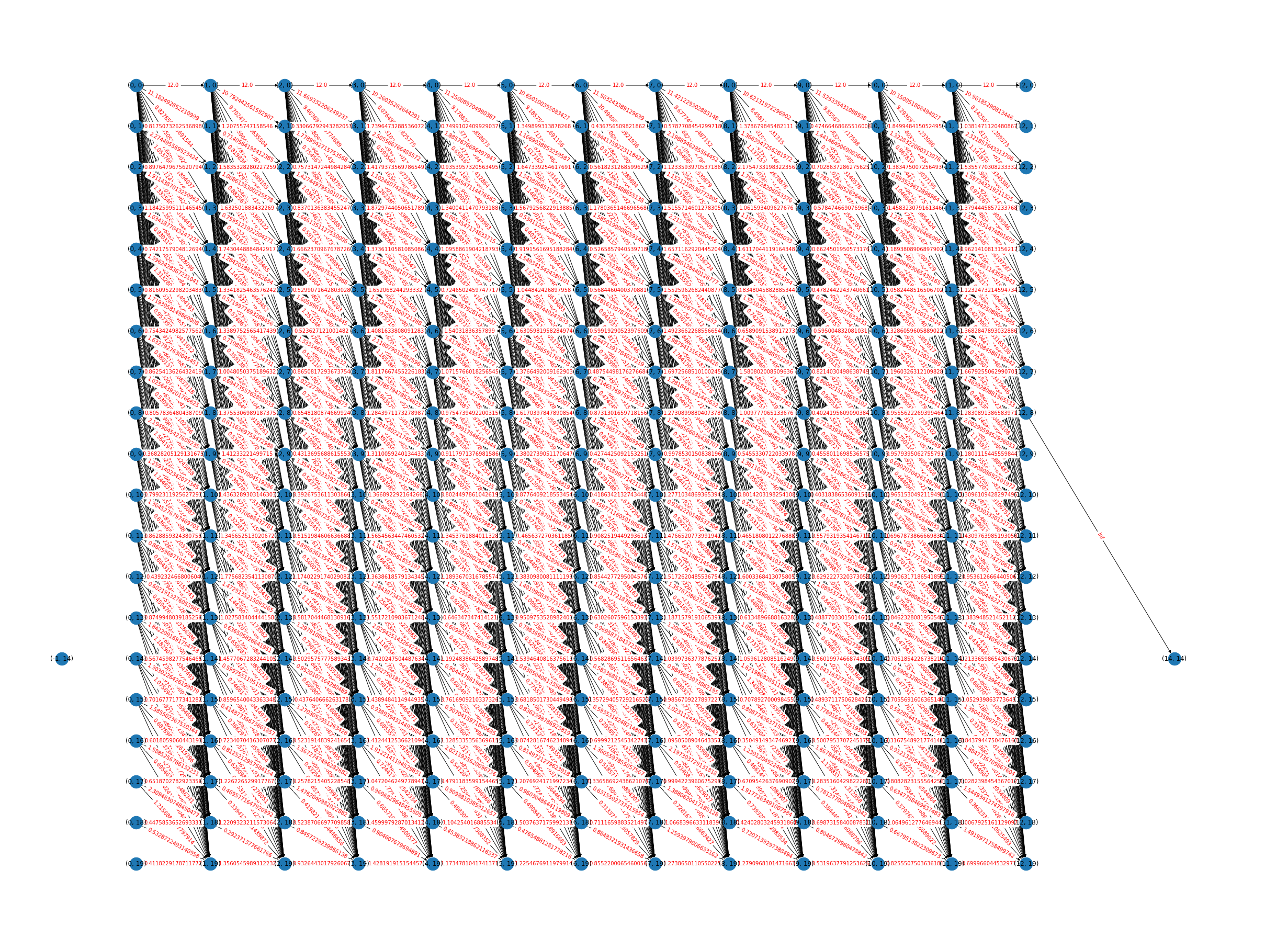}
    \caption{The flow network for the decoder only architecture GPT-2 for the example in Figure~\ref{fig:GPT2:flow:changes}.}
    \label{fig:decodernetwork}
\end{figure*}
\end{document}